\pdfoutput=1
\documentclass[11pt]{article}

\usepackage[utf8]{inputenc}
\usepackage{amsmath,amssymb,amsthm}
\usepackage{microtype}
\usepackage{hyperref}
\usepackage{aliascnt}
\usepackage[nameinlink,noabbrev]{cleveref}
\usepackage{geometry}
\usepackage{enumitem}
\usepackage{xcolor}
\usepackage{array}
\usepackage{tikz}
\usetikzlibrary{decorations.pathreplacing}

\newcommand{\OPT}{\mathsf{OPT}}
\newcommand{\Wil}{\mathsf{W}}
\newcommand{\New}{\mathsf{New}}
\newcommand{\col}{\operatorname{col}}
\newcommand{\shuffle}{\mathbin{\sqcup\!\sqcup}}

\newtheorem{theorem}{Theorem}
\newaliascnt{lemma}{theorem}
\newtheorem{lemma}[lemma]{Lemma}
\newtheorem{proposition}[theorem]{Proposition}
\newtheorem{corollary}[theorem]{Corollary}
\newtheorem{conjecture}[theorem]{Conjecture}

\theoremstyle{definition}
\newtheorem{definition}[theorem]{Definition}
\newtheorem{problem}{Problem}

\theoremstyle{remark}
\newtheorem{remark}[theorem]{Remark}

\aliascntresetthe{lemma}
\crefname{lemma}{lemma}{lemmas}
\Crefname{lemma}{Lemma}{Lemmas}

\title{Bolzano: Case Studies in LLM-Assisted Mathematical Research}
\author{
\begin{tabular}{c}
Martin Balko$^{4}$ \qquad
Jan Grebík$^{1}$ \qquad
Pavel Hubáček$^{1,2}$ \qquad
Martin Koutecký$^{1}$ \\[0.5ex]
Matěj Kripner$^{3}$ \qquad
Václav Rozhoň$^{1}$ \qquad
Robert Šámal$^{1}$ \qquad
Adrián Zámečník$^{1}$ \\[1ex]
\small$^{1}$Computer Science Institute, Charles University \\
\small$^{2}$Institute of Mathematics, Czech Academy of Sciences \\
\small$^{3}$Institute of Formal and Applied Linguistics, Charles University \\
\small$^{4}$Department of Applied Mathematics, Charles University 
\end{tabular}
}
\date{}
\begin{document}

\maketitle

\begin{abstract}
\noindent
We report new results on eight problems in mathematics and theoretical computer science, produced with the assistance of Bolzano--an open-source multi-agent LLM system.
Bolzano orchestrates rounds of interaction between parallel prover agents and a verifier agent while maintaining a~persistent knowledge base that is carried across rounds.
Classified using the significance–autonomy taxonomy of Feng et al.~\cite{aletheia2026}, six of the eight results reach the level of publishable research, and five of the eight were produced essentially autonomously by Bolzano.
Our results provide evidence that LLMs can contribute meaningfully to mathematical research, complementing recent reports by Bubeck et al.~\cite{openai2025gpt5}, Woodruff et al.~\cite{google2025gemini}, and others.
\end{abstract}

\section{Introduction}

The integration of large language models into mathematical research has moved rapidly from speculation to documented practice. We highlight two recent reports that have established this as a real phenomenon (see Related Work for many more results in this direction):

\begin{itemize}
    \item Woodruff, Mirrokni et al.~\cite{google2025gemini} present a collection of case studies in which Google's Gemini Deep Think solved open problems across mathematics, theoretical computer science, and physics. 
    \item Bubeck et al.~\cite{openai2025gpt5} document experiments in which OpenAI's GPT-5 contributed to research in mathematics, physics, and computer science. 
\end{itemize}

This paper provides more evidence in this direction. Bolzano \cite{bolzano} is an open-source AI system that assists with mathematical research and is available to interested researchers.
The system runs research on a problem using several rounds of interaction between several provers and a verifier. Each prover and verifier is implemented by state-of-the-art LLMs (GPT, Gemini, Claude, or others).
Its architecture is described in \Cref{sec:bolzano}. We use Bolzano to provide new results for eight problems in mathematics and theoretical computer science.

\paragraph{Strengths of Bolzano}

Across our case studies, Bolzano was most useful as a generator of mathematically meaningful intermediate research moves: it was particularly good at finding counterexamples and obstructions, proposing concrete constructions and gadgets, or extending a known base case or simpler proof template to a more general statement. This mirrors the findings in \cite{google2025gemini,openai2025gpt5}. In one case, Bolzano autonomously provided a hardness proof and a polynomial-time algorithm for a restricted class, and corrected a small error in the human-provided problem formulation.

\paragraph{Limitations}

This work has several limitations. First, we do not provide a thorough comparison between Bolzano and single-session chatbot interactions; while such a comparison is feasible in principle, it is difficult to carry out rigorously in a case study with eight problems, and the advantage of the multi-agent architecture over direct use of the same underlying models remains unquantified.
Second, it is hard to cleanly disentangle the human and AI contributions; even selecting the right problem to work on is a nontrivial human input.

\paragraph{Solved problems}
We document eight problems in mathematics and theoretical computer science that Bolzano has resolved or made progress on. This was done either autonomously or through collaboration with domain-expert mathematicians and computer scientists. For each of the problems, we include discussion. The corresponding formal results are proved either in the appendices of this paper or in separate papers.

\begin{table}[!t]
\centering
\small
\begin{tabular}{@{} m{3.7cm} | >{\centering\arraybackslash}m{3.7cm} | >{\centering\arraybackslash}m{3.7cm} | >{\centering\arraybackslash}m{4cm} @{}}
 & \textbf{H} (Primarily human) & \textbf{C} (Collaboration) & \textbf{A} (Autonomous) \\
\hline
\textbf{Negligible novelty} & & & \\
\hline
\textbf{Minor novelty} & &  & \shortstack[l]{Partitioning (\Cref{sec:partition})\\Tilings (\Cref{sec:tiling})} \\
\hline
\textbf{Publishable research} & Heaps (\Cref{sec:heap}) & \shortstack[l]{PWPP (\Cref{sec:pwpp})\\KZG (\Cref{sec:chopin})} & KKOS Optim.\ (\Cref{sec:opdiff})\newline CCE (\Cref{sec:gameTheory})\newline BST Interleaving (\Cref{sec:wilber}) \\
\hline
\textbf{Major advance} & & & \\
\hline
\shortstack[l]{\textbf{Landmark}\\\textbf{breakthrough}} & & & \\
\end{tabular}
\caption{Classification of our results by significance and autonomy, following the taxonomy proposed by Feng et al.~\cite{aletheia2026}. 
Autonomy ranges from H (primarily human, secondary AI input) through C (human--AI collaboration) to A (essentially autonomous).
}
\label{tab:classification}
\end{table}

\begin{enumerate}
    \item \textbf{Complexity theory:} We construct an oracle separation showing that PWPP is not closed under adaptive Turing reductions (Section~\ref{sec:pwpp}). The human input was the question, and after Bolzano returned a proof, a human researcher suggested to keep the main structure (the same choice of problem showing that PWPP is not closed) but try a different strategy for the analysis, after which Bolzano essentially finished the proof. 
    \item \textbf{Additive combinatorics:} We provide an example of a tiling of $\mathbb{R}^2$ by a single tile that shows a limit on the technique from a forthcoming paper \cite{DGGM_tiling_in_progress} (Section~\ref{sec:tiling}).
    \item \textbf{Cryptography:} We prove special soundness for multi-polynomial, multi-point KZG batching in the standard model (Section~\ref{sec:chopin}). The human researchers proposed a variant of the protocol and provided a simpler base-case proof as context. Bolzano found the full proof in six rounds, producing a lengthy but essentially complete formal argument.
    \item \textbf{Data structures:} We prove that binary search trees have a natural interleaving property, up to $O(\log\log n)$. Bolzano produced the proof autonomously (Section~\ref{sec:wilber}).
    \item \textbf{Combinatorics:} We disprove a conjecture on function-preimage partitioning from the KAMAK 2020 workshop \cite{kral2020}. Moreover, Bolzano establishes a corrected hypothesis and proves some bounds for it (Section~\ref{sec:partition}).
    \item \textbf{Computational complexity:} We determine the complexity of an optimization problem over the KKOS cultural dynamics model~\cite{kkos} (Section~\ref{sec:opdiff}). The decision version is NP-complete on general graphs; on forests, we give a polynomial-time algorithm. Bolzano worked autonomously, also correcting a small error in the original problem formulation.
    \item \textbf{Game Theory:} We show that, for every $\varepsilon \in (0,1/2)$, there are $n$-player normal-form games where every $\varepsilon$-coarse correlated equilibrium has support of size at least  $\Omega\left(\frac{\log{n}}{\varepsilon^2\log{(1/\varepsilon)}}\right)$, answering negatively an open problem posed by Babichenko, Barman, and Peretz~\cite{babichenko14}.
    This result was proved by Bolzano in four rounds, after a human researcher pointed out the direction of the proof (Section~\ref{sec:gameTheory}).
    \item \textbf{Data structures:} We prove that two beyond worst-case properties of heaps are equivalent (Section~\ref{sec:heap}). This was first proven by a human researcher in a (yet unpublished) project. Bolzano independently reproved the equivalence and independently came up with a stronger statement and its proof.
\end{enumerate}

In general, our experience parallels the findings of~\cite{google2025gemini} and~\cite{openai2025gpt5} in several respects. LLMs excel at generating proof candidates, constructing counterexamples, and making cross-domain connections. In early 2026, human guidance remains relevant for problem selection, high-level strategy, and verification. 

\paragraph{Related work}

Recent work has begun to document the use of LLMs for mathematical and scientific research. GPT-5 has been shown to contribute to research in mathematics, physics, and other sciences, including solving open problems and improving constants~\cite{openai2025gpt5}. Similar results have been reported with Gemini~\cite{google2025gemini}. Aletheia~\cite{aletheia2026} is a mathematics research agent based on a Generator--Verifier--Reviser loop that autonomously solved research-level problems, including 6 out of 10 problems in the First Proof challenge~\cite{aletheia_firstproof2026}. First Proof~\cite{firstproof2026} is a benchmark of 10 research-level problems from unpublished work of 11 mathematicians; OpenAI~\cite{openai_firstproof2026} reported at least 5 likely correct solutions. AlphaEvolve~\cite{novikov2025ai} has been applied to 67 problems~\cite{georgiev2025}, mostly in constant optimization and construction search, with expert hints improving efficiency without clearly changing the ceiling of performance. FunSearch~\cite{romeraParedes2024funsearch} pairs an LLM with an evaluator in an evolutionary loop and has produced new results in combinatorics and algorithms. 

Many systems pursue automated theorem proving with formal verification. AlphaGeometry~\cite{trinh2024alphageometry} solves olympiad geometry problems via a neuro-symbolic engine. Many recent systems target the Lean proof assistant, including AlphaProof~\cite{hubert2025alphaproof}, which achieves IMO medal-level performance via reinforcement learning; Hilbert~\cite{hilbert2025}, which combines informal proof sketches with formal verification; Aleph Prover~\cite{alephprover2025}, a formal theorem prover that leads PutnamBench; Ax-Prover~\cite{axprover2025}, which targets quantum physics and abstract algebra; and autoformalization of a 130{,}000-line topology textbook at low cost~\cite{urban2025}. We view formal verification as complementary to our approach: Bolzano relies on informal proof generation followed by expert verification, which currently offers broader coverage but weaker guarantees. A large-scale evaluation of over 5{,}000 LLM-generated proofs across 1{,}000 problems~\cite{dekoninck2025} finds substantially stronger performance in informal than in formal proof generation.

Multi-agent architectures for mathematical reasoning have been explored in several directions, including multiagent debate, step-level verification, iterative self-refinement, and LLM scaffolding for problem solving~\cite{du2023debate,guo2025parastep,huangyang2025,madaan2023selfRefine}.
In late 2025, open Erd\H{o}s problems emerged as a testing ground for LLM mathematical capabilities, with models producing original proofs of previously unsolved problems~\cite{putterman2026,feng2026erdos}; however, the problems solved so far are generally amenable to straightforward techniques~\cite{feng2026erdos}. 

\paragraph{Roadmap}
\Cref{sec:bolzano} describes the architecture and design of Bolzano. \Cref{sec:problems} presents each of the solved problems, including the problem statement, context, a discussion of the human--AI collaboration, and a link to the Bolzano transcript. For problems whose proofs are not contained in a separate paper, the full proofs appear in the appendices. All proofs have been verified by domain experts.

\section{Bolzano}
\label{sec:bolzano}

Bolzano is a research tool available at \url{https://bolzano.app}.
It provides an automated pipeline consisting of \emph{prover}, \emph{verifier}, and \emph{summarizer agents} designed to iteratively investigate mathematical research problems. Throughout this section, by \emph{agent} we mean a callable function that makes a preconfigured request to an LLM with a custom prompt specifying a research persona, tasks, and goals.

\paragraph{Overall architecture}
We call one iteration of the Bolzano pipeline a \emph{research round}. It involves running $n$ parallel \emph{prover agents}, followed by one \emph{verifier agent}, and concluded by a \emph{summarizer agent}. Research rounds run sequentially; between rounds, the state of the investigation is preserved in three human-readable files.

\paragraph{Agents.}
The \emph{prover agents} are tasked with coming up with proof ideas, finding counterexamples, proving special cases, identifying mistakes, and writing proofs. 
The \emph{verifier agent} checks the work of the provers---identifying errors and unjustified steps, combining ideas into viable solutions---and is the sole agent that decides what gets written into the knowledge base files. The \emph{summarizer agent} produces a concise summary of each research round for the user and for subsequent agents.

\paragraph{Knowledge base files.}
In between research rounds, we maintain three files that serve as a persistent knowledge base. The \emph{notes} file aggregates insights, failed approaches, conjectures, and simplified proofs. The \emph{proofs} file maintains rigorous, fully detailed proofs. The \emph{output} file provides a short summary of the current status for the human researcher. All agents read from these files, but only the verifier can write to them. 

\paragraph{Model diversity.}
Bolzano allows the user to select a different LLM for each prover agent. Since models differ in training data, this seems to generate a wider variety of approaches. This approach is also reported to counteract self-preference bias~\cite{panickssery2024}.

\paragraph{Human guidance.}
Between research rounds, the user may provide additional instructions to steer the investigation. In our experience, human guidance often leads to stronger results---for example, the PWPP result (\Cref{sec:pwpp}) was obtained after an expert advised Bolzano to try a different strategy. %

\section{Solved Problems}
\label{sec:problems}

This section contains the eight selected problems solved by Bolzano. For each problem, we add discussion about what parts of the research pipeline have been done by a human, and which parts by Bolzano. 
We always add a link to the Bolzano system containing the proof. However, Bolzano-generated proofs are not intended to be publication-ready. We thus always either add expert-verified proof to appropriate appendix, or link to a paper containing the proof.  

\subsection{Black-Box Separation of Adaptive and Non-Adaptive PWPP -- Pavel Hub\'a\v{c}ek}
\label{sec:pwpp}
The complexity class PWPP (Polynomial Weak Pigeonhole Principle)~\cite{jerabek2016} captures collision finding within TFNP.
Its canonical complete problem \textsc{Collision} asks: given a shrinking circuit
$C:\{0,1\}^n\to\{0,1\}^{n-1}$, find distinct $x_1,x_2\in\{0,1\}^n$ such that $C(x_1)=C(x_2)$.
PWPP consists of all total search problems many-one reducible to \textsc{Collision}.

Je\v{r}\'abek~\cite{jerabek2016} showed that PWPP is closed under \emph{non-adaptive} Turing reductions,
i.e.\ $P^{\|\mathrm{PWPP}}=\mathrm{PWPP}$: solving $k$ independently prepared PWPP instances reduces to a single collision query.
A basic structural question for TFNP subclasses is whether they remain closed under \emph{adaptive} Turing reductions,
where later oracle queries may depend on earlier answers~\cite{bussjohnson2012}.
For several classes (e.g.\ PLS, PPA, PPAD) adaptive and non-adaptive oracle access coincide~\cite{bussjohnson2012}, while for the related class PPP (Polynomial \emph{strong} Pigeonhole Principle),
Fleming et al.~\cite{fleming2024} proved a black-box separation showing it is not Turing-closed.

Together with Bolzano, in \cite{pwpp_paper} we resolve the analogous question for PWPP in the black-box setting by introducing a natural adaptive task
\textsc{NestedCollision}, suggested by Bolzano, which requires two dependent collision-finding steps.

After being given the problem, Bolzano produced the core construction and proof in four rounds of interaction. While the initial proof had flaws, after being instructed to adopt a different high-level strategy,\footnote{Concretely: ``Discard the case analysis strategy. Instead, prove that $\Pi \not\in PWPP^{\mathcal O}$ by establishing that, with overwhelming probability, the reduction circuit contains "useless" collisions that do not reveal a solution to $\Pi$.''} Bolzano delivered a mostly complete formal proof in four additional rounds, with only minor typographical errors and no significant logical gaps. 

\paragraph{Formal results.}

The theorem proven in~\cite{pwpp_paper} is the following. 

\begin{theorem}[Black-box PWPP is not Turing-closed~\cite{pwpp_paper}]
\label{thm:pwpp-intro}
In the decision-tree model, the search problem $\textsc{NestedCollision}$ admits no shallow $\textsc{Collision}$-formulation.
Therefore, black-box $\mathrm{PWPP}$ is not closed under adaptive Turing reductions.
\end{theorem}

The Bolzano research transcript is available at~\cite{bolzano_pwpp}. The complete proof appears in \cite{pwpp_paper}.

\subsection{Structural Results on Multi-Slope Tilings -- Jan Grebík}
\label{sec:tiling}

\paragraph{Motivation and result}

The problem originated in the study of \emph{translational monotilings of $\mathbb{R}^d$}, or less generally $\mathbb{Z}^d$, where the setup is the following.
We are given a measurable set $\Omega\subseteq \mathbb{R}^d$ of finite positive Lebesgue measure and want to understand if there is a (necessarily uniformly discrete) set of translates $T\subseteq \mathbb{R}^d$ such that
$$\Omega\oplus T=\mathbb{R}^d,$$
where $\Omega\oplus T$ means that the collection of translates $\{\Omega+t\}_{t\in T}$ are disjoint up to null sets and cover all of $\mathbb{R}^d$.
In this case, $T$ is called a \emph{tiling of $\mathbb{R}^d$ by $\Omega$}, and $\Omega$ is called a \emph{tile}.
The definitions for translational monotilings of $\mathbb{Z}^d$ are analogous.
Much of the recent development in the area \cite{bhattacharya,GreenfeldTaoStructureZd,greenfeldtao,GT2} have been driven by the so-called \emph{periodic tiling conjecture}.

\begin{conjecture}[Periodic tiling conjecture (PTC), \cite{lagariaswang}]
    Let $\Omega\subseteq \mathbb{R}^d$ be a tile.
    Then there is a tiling $T$ of $\mathbb{R}^d$ by $\Omega$ that is periodic, that is, $\{\gamma\in \mathbb{R}^d:\gamma+T=T\}$ contains a lattice.
\end{conjecture}

Unlike in the case of general tiling problems, where more tiles are allowed, PTC holds in $\mathbb{Z}^2$ which was proven by Bhattacharya \cite{bhattacharya}.
On the other hand Greenfeld and Tao \cite{greenfeldtao} showed that PTC fails in $\mathbb{Z}^d$ (and thus $\mathbb{R}^d$) for large enough $d$.
In $\mathbb{R}^2$ Kenyon \cite{kenyon} showed that PTC holds for tiles that are topological disks, but the general case is widely open.
Together with de Dios, Greenfeld and Madrid we investigated the case of tiles that are polygonal sets (possibly disconnected with holes) with edges being axes parallel that may have irrational lengths.
We obtain the following general statement, providing a weak form of PTC.

\begin{theorem}[\cite{DGGM_tiling_in_progress}]\label{thm:weakPTC}
    Let $\Omega\subseteq \mathbb{R}^2$ be an axes parallel polygonal tile.
    Then there is $k\in \mathbb{N}$ and a tiling $T=T_1\sqcup \dots \sqcup T_k$ of $\mathbb{R}^2$ by $\Omega$ such that after possibly swapping the vertical and horizontal axes the following holds:
    \begin{enumerate}
        \item $T_i$ is periodic for every $1\le i\le k$,
        \item $\Omega\oplus T_i$ is a union of cosets of $\mathbb{R}(0,1)$,
        \item there is $\gamma\in \mathbb{R}^2\setminus \{0\}$ such that $T+\gamma=T$, more specifically, if $k>1$, then $\gamma$ is of the form $\gamma=(0,\alpha)$ for some $\alpha>0$.
        \end{enumerate}
\end{theorem}

Note that if $k=1$ (or in the discrete case $\mathbb{Z}^2$), then (1) above would already imply that $T$ is periodic, thus establishing the PTC for $\Omega$.
Our current techniques do not seem to give any information about the possible relation between $T_i$'s or the sets of the form $\Omega\oplus T_i$.
This leads to the question of whether there exist interesting tilings as in \Cref{thm:weakPTC}.
Without any restriction, the answer to this question is trivial, as one might consider the lattice tiling of $\mathbb{R}^2$ by a unit square $[0,1]\times [0,1]$ and construct for any $k\in \mathbb{N}$ a tiling $T$ as above by shifting different columns.
In order to avoid this trivial case we need a definition.

\begin{definition}\label{def:Column}
    We say that $\Omega\subseteq \mathbb{R}^2$ \emph{tiles a column} (or is a \emph{column tile}), if there is $S\subseteq \mathbb{R}(0,1)$ such that $\Omega\oplus S$ is a union of cosets of $\mathbb{R}(0,1)$.
\end{definition}

Bolzano produced non-column examples for Theorem~\ref{thm:weakPTC}, where $T=T_1\sqcup T_2$ and each part of the decomposition is periodic with different lattices.
This is a first step towards understanding additional flexibility that tilings of $\mathbb{R}^2$ enjoy compared to tilings of $\mathbb{Z}^2$.

\begin{theorem}\label{thm:Main_Tiling}
    For every irrational $\alpha\in (2/3,1)$ there is an axis parallel polygonal tile $\Omega_\alpha\subseteq \mathbb{R}^2$ that is not a column tile together with a tiling $T_{\alpha}=T_{1,\alpha}\sqcup T_{2,\alpha}$ of $\mathbb{R}^2$ by $\Omega_\alpha$ such that the following holds:
    \begin{enumerate}
        \item $T_{1,\alpha}$ is $(2\mathbb{Z})\times \mathbb{Z}$ periodic,
        \item $T_{2,\alpha}$ is periodic with lattice $\{(2k,k\alpha+\ell):k,\ell\in\mathbb{Z}\}$,
        \item $\Omega_\alpha\oplus T_{i,\alpha}$ is a union of cosets of $\mathbb{R}(0,1)$ for $1\le i\le 2$,
        \item $T_\alpha$ is $(0,1)$-periodic.
        \end{enumerate}
\end{theorem}

The proofs appear in \Cref{app:tiling}.
The Bolzano research transcript is available at~\cite{bolzano_tiling}.

\begin{remark}
    Let us mention that each of the tiles $\Omega_\alpha$ tile a horizontal column (that is, when we swap the horizontal and vertical axes, then the modified tile $\Omega'_\alpha$ tiles a column), and the projection of $\Omega_\alpha\oplus T_i$ to the first coordinate is equal to $(2\mathbb{Z}+(i-1))+[0,1]$.
    It seems that finding examples that would not have such a structure, or that would admit tilings that split as $T=T_1\sqcup T_2\sqcup T_3$ with the period of $T_3$ parametrized by $\beta>0$ such that both $\beta,\beta-\alpha$ are irrational, is significantly more complicated, and Bolzano was not able to find an example nor to prove a general obstruction result for any of these conditions.
\end{remark}

\subsection{Special Soundness for Univariate KZG Batching -- Pavel Hub\'a\v{c}ek}
\label{sec:chopin}
Polynomial Commitment Schemes (PCS) are a core building block of modern zero-knowledge proofs (zk-SNARKs).
To minimize proof size and verification costs, practical systems often rely on batching techniques that allow a prover to aggregate evaluations of multiple polynomials at multiple points into a single proof.
For the popular univariate KZG commitment scheme~\cite{KateZG10}, existing multi-polynomial, multi-point batching protocols (e.g., \cite{BDFG20}) have predominantly been analyzed only in idealized settings, limiting the assurance for their practical deployments.
Proving their knowledge soundness in the standard model under falsifiable assumptions is a notoriously difficult task; the first prior standard-model analysis was strictly limited to the simpler case of batching evaluations of many polynomials at a \emph{single} evaluation point~\cite{LPS25}.

In the development of CHOPIN~\cite{chopin_paper}, an optimal pairing-based multilinear PCS, we required a fully rigorous standard-model security proof for multi-polynomial, multi-point batching that was not known.
The human researchers proposed a variant of the KZG batch proof from~\cite{BDFG20} to simplify the task of proving its \emph{special soundness in the standard model}, the core task towards a complete proof of knowledge soundness of the scheme.
However, the human researchers did not have any rigorous proof.

Together with Bolzano, we resolved this gap.
To assist the system, the input also contained a proof of the more basic theorem establishing the special soundness for batching KZG evaluation proofs for many polynomials at a single evaluation point from~\cite{LPS25}, streamlined by the human researchers.
After being given the problem, Bolzano found the proof of special soundness for the proposed variant of univariate KZG batching in six rounds of interaction.
The proof was lengthy and technical, but, besides minor edits in notation and presentation, it was complete and formal as verified by the authors.

\paragraph{Formal results.}

The lemma proven in~\cite{chopin_paper} is the following.

\begin{lemma}[Special soundness of multi-polynomial, multi-point KZG batching~{\cite[Lemma 3]{chopin_paper}}]
\label{lem:kzg-batching}
Let $m, M, M' \in \mathrm{poly}(\lambda)$. Let $T = \bigcup_{t=1}^m S_t$ and let $L = |T| + M$.
Assume that KZG for degree less than $M$ has $M$-special soundness in the following sense.
From any set of $M' \ge M$ accepting KZG opening transcripts for the same commitment $C$ at $M'$ distinct points
$a_1, \dots, a_{M'}$ with $a_j \neq \tau$, one can extract a polynomial $p \in \mathbb{F}[X]^{<M}$ such that
$C = [p(\tau)]_1$ and $p(a_j) = y_j$ for all $j \in [M']$.
Then there exists an extractor that, given as input any product-structured accepting $(m, L)$-tree
$\mathcal{T} = (\mathsf{tr}_{ij})_{i \in [m],\, j \in [L]}$
for the batch evaluation protocol of Figure~7 of~\cite{chopin_paper}, outputs polynomials
$p_1, \dots, p_m \in \mathbb{F}[X]^{<M}$ such that $C_t = [p_t(\tau)]_1$ for all $t \in [m]$ and
$p_t(z) = \eta_z$ for all $t \in [m]$ and all $z \in S_t$.
Equivalently, the batch evaluation protocol is $(m, L)$-special sound with respect to product-structured
accepting transcript trees.
\end{lemma}

The Bolzano research transcript is available at~\cite{bolzano_chopin}. The complete proof appears in \cite{chopin_paper}.

\subsection{Interleaving BST Access Sequences and Wilber's First Bound -- Václav Rozhoň}
\label{sec:wilber}

The setting is the classical online binary search tree (BST) problem, in which the cost of an execution equals the total number of nodes touched while serving a sequence of key accesses. The dynamic optimality conjecture of Sleator and Tarjan~\cite{sleator} asks whether the splay tree is $O(1)$-competitive with the offline optimum $\OPT(S)$. The best known competitive ratio is $O(\log\log n)$, achieved by Tango trees~\cite{tango}, whose analysis is driven by Wilber's first lower bound~\cite{wilber}.

A natural question that came up in the work on the dynamic optimality conjecture is what is the behavior of the BST cost under interleaving. For access sequences $X=(x_1,\dots,x_m)$ and $Y=(y_1,\dots,y_m)$ over $[n]$, write
\[
X \shuffle Y \;:=\; (x_1,y_1,x_2,y_2,\dots,x_m,y_m).
\]
A natural conjecture is
\[
\OPT(X \shuffle Y) \;\stackrel{?}{=}\; O\bigl(\OPT(X) + \OPT(Y) + m + n\bigr).
\]
Bolzano proved the exact analogue of this conjecture for Wilber's first bound.

\begin{theorem}
\label{thm:wilber-merge}
Let $Z$ be any two-colored sequence over $[n]$, with monochromatic subsequences $R$ and $B$. Then
\[
\Wil(Z) \;\le\; 3\,\Wil(R) + 3\,\Wil(B) + |Z|.
\]
In particular, for access sequences $X,Y$ of length $m$,
\[
\Wil(X \shuffle Y) \;\le\; 3\,\Wil(X) + 3\,\Wil(Y) + 2m.
\]
\end{theorem}

Combining \Cref{thm:wilber-merge} with Wilber's lower bound and the Tango-tree upper bound $\mathrm{BST\text{-}cost}(S) = O\bigl((\Wil(S)+|S|)\log\log n + n\bigr)$ yields the following corollary, which says that the BST interleaving conjecture holds up to the standard $O(\log\log n)$ slack separating Wilber's first bound from the offline optimum.

\begin{corollary}
\label{cor:wilber-merge}
Let $X,Y$ be BST access sequences of length $m$ over $[n]$ that can be served with cost $C_1$ and $C_2$ respectively. Then $X \shuffle Y$ can be served by a BST with cost
\[
O\bigl((C_1 + C_2 + m + n)\log\log n\bigr).
\]
\end{corollary}

Both \Cref{thm:wilber-merge} and \Cref{cor:wilber-merge} are proved in \Cref{app:wilber-merge}. The Bolzano research transcript is available at~\cite{bolzano_wilber}.

\subsection{Partitioning under Function Preimage Constraints -- Robert Šámal}
\label{sec:partition}

KAMAK is a Czech problem-solving workshop where participants propose and collect open
problems in combinatorics and discrete mathematics. The present problem appears as Problem~1 in the collection from the 2020 edition~\cite{kral2020}.
The problems are suggested by the participants and are of various difficulty, but always meant as potentially interesting research problems. 

Consider sets $E$ and $F$, together with functions
$f_1,\dots,f_k:E\to F$ satisfying \emph{pointwise distinctness}, meaning that
\[
f_i(x)\neq f_j(x)
\qquad\text{for every } x\in E \text{ and every } i\neq j.
\]
Conjecture suggested by C. Feghali asked whether a condition on fibre sizes forces a certain bounded partition.
(Pointwise distinctness is obviously necessary.) 
The case $k=2$ was known and had applications in digraph coloring \cite{feghali_case2}. 

\begin{conjecture}[Original conjecture; the case $k=2$ is known]
If for every $z\in F$ there exists $t\in\{1,\dots,k\}$ with
$|f_t^{-1}(z)|\le n$, then $E$ can be partitioned into $2n+1$ parts
$E_1,\dots,E_{2n+1}$ such that
\[
f_p(E_i)\cap f_q(E_i)=\emptyset
\qquad\text{for every } i \text{ and every } p<q.
\]
\end{conjecture}

A convenient way to view the problem is through the \emph{conflict graph} on vertex set $E$,
in which distinct $x,y\in E$ are adjacent whenever
$f_p(x)=f_q(y)$ for some $p\neq q$. Then a partition with the required disjointness
property is exactly a proper coloring of this graph. Bolzano's first contribution was to
observe that for $k\ge 3$ the original hypothesis can be satisfied vacuously: a dummy
function may have empty fibers and thus the remaining functions are not controlled. 

This leads to the following negative result.

\begin{theorem}[Counterexample to the original conjecture]
\label{thm:disproof}
For every $n\ge 1$ and every $M\ge 1$, there exist sets $E,F$ and functions
$f_1,f_2,f_3:E\to F$ such that:
\begin{enumerate}
    \item $f_i(x)\neq f_j(x)$ for all $x\in E$ and all $i\neq j$;
    \item for every $z\in F$ there exists $t\in\{1,2,3\}$ with $|f_t^{-1}(z)|\le n$;
    \item every partition of $E$ satisfying
    \[
    f_p(E_i)\cap f_q(E_i)=\emptyset
    \qquad\text{for all } i \text{ and all } p<q
    \]
    requires more than $M$ parts.
\end{enumerate}
\end{theorem}

The construction realizes the conflict graph as a shift graph, whose chromatic number is unbounded. 
The full proof appears in Appendix~\ref{app:function-partition}.

The counterexample suggests that the right assumption is not to control single fibers, but to
control them \emph{pairwise}. This leads to the following notion, \emph{also suggested by Bolzano.}

\begin{definition}[Pairwise $n$-boundedness]
The functions $f_1,\dots,f_k:E\to F$ are \emph{pairwise $n$-bounded} if for every
$z\in F$ and every pair $p\neq q$,
\[
\min\bigl(|f_p^{-1}(z)|,\;|f_q^{-1}(z)|\bigr)\le n.
\]
\end{definition}

Under this stronger hypothesis one gets a positive result. 

\begin{theorem}[Pairwise $n$-boundedness implies bounded partition]
\label{thm:pairwise-intro}
Assume pointwise distinctness and pairwise $n$-boundedness. Then $E$ can be partitioned
into $2nk(k-1)+1$ parts $E_1,\dots,E_{2nk(k-1)+1}$ such that
\[
f_p(E_i)\cap f_q(E_i)=\emptyset
\qquad\text{for every } i \text{ and every } p<q.
\]
\end{theorem}

The proof orients each conflict toward the endpoint whose side of the witnessing equality
comes from a small fiber, thereby obtaining a bounded-indegree orientation of the conflict
graph. This implies bounded degeneracy and hence bounded chromatic number.
Appendix~\ref{app:function-partition} contains the full proof, together with a stronger
bound under uniform $n$-boundedness and complementary lower-bound constructions.

The Bolzano research transcript is available at~\cite{bolzano_partition}.

\subsection{Complexity of Optimization in KKOS Cultural Dynamics -- Martin Kouteck\'y}
\label{sec:opdiff}

Kempe, Kleinberg, Oren, and Slivkins~\cite{kkos} introduced a model of cultural dynamics in which agents on a social network update their opinions to reduce disagreement with neighbors. In their \emph{local model}, the equilibrium condition requires that neighboring agents in the support of a distribution experience equal ``mass'' -- that is, the total weight in their closed neighborhood is the same.

It is natural (especially motivated by bribery-type viewpoints, see~\cite{JLOGCOM}) to consider the problem of finding a closest equilibrium $x$ to a given (arbitrary) distribution $y$.
Formally, given an undirected graph $G=(V,E)$ with adjacency matrix~$A'$, set $A=A'+I$ (adding self-loops), and given an initial distribution $y \in \mathbb{R}_{\ge 0}^V$ with $\|y\|_1=1$ and a cost vector $c \in \mathbb{R}_{\ge 0}^V$, the task is to find a distribution $x \in \mathbb{R}_{\ge 0}^V$ with $\sum_v x_v=1$ minimizing $\sum_v c_v |x_v - y_v|$ subject to the constraint that for every edge $uv \in E$, if $x_u, x_v > 0$ then $(Ax)_u = (Ax)_v$.

Bolzano worked on this problem autonomously over three rounds. In Round~1, Bolzano found a clean NP-hardness proof via a reduction from \textsc{Clique}: adding a universal vertex forces any feasible support to be a clique, and the $\ell_1$ cost becomes monotone in the clique size. In Round~2, Bolzano established membership in NP (via an LP-based polynomial certificate), corrected a small error in the original problem formulation,\footnote{The task statement asserted that for a fixed support~$S$, one can compute a cost-minimizing vector by linear programming. This is not true in the straightforward LP formulation because the exact-support constraint involves strict inequalities, so the feasible region can be open. Bolzano provided a counterexample and a correct reformulation using a max-margin LP.} and proved structural results for chordal graphs and forests. In Round~3, Bolzano designed a polynomial-time $O(n^2)$ algorithm for forests via dynamic programming on dissociation sets.

\paragraph{Formal results.}

\begin{theorem}[NP-completeness]
\label{thm:opdiff-npc-intro}
Under the standard binary encoding of rational inputs, the decision version of the KKOS optimization problem is NP-complete. NP-hardness holds even with unit costs, positive rational~$y$, and a universal vertex. The optimization problem is NP-hard.
\end{theorem}

\begin{proposition}[Forest characterization]
\label{prop:opdiff-forest-intro}
On a forest, the feasible supports are exactly the dissociation sets, i.e., the vertex sets $S$ such that $G[S]$ has maximum degree at most~$1$.
\end{proposition}

\begin{theorem}[Polynomial-time algorithm for forests]
\label{thm:opdiff-forest-algo-intro}
If $G$ is a forest, the optimization problem can be solved in $O(n^2)$ time. The forest-restricted decision problem is in~$P$.
\end{theorem}

\begin{remark}
On chordal graphs, connected feasible supports must be cliques (\Cref{prop:chordal-clique}), so feasible supports are exactly disjoint unions of cliques.
\end{remark}

The proofs appear in \Cref{app:opdiff}.
The Bolzano research transcript is available at~\cite{bolzano_opdiff}.

\subsection{Estimating the support size of $\varepsilon$-coarse correlated equilibria -- Martin Balko}
\label{sec:gameTheory}

Nash's theorem, a fundamental result in game theory, says that every normal-form game $G$ has at least one stable state, called \emph{Nash equilibrium}.
However, all known proofs of this classical result are non-constructive, and there is strong evidence suggesting that there might not be a polynomial-time algorithm to compute Nash equilibria, as this task is PPAD-complete even for 2-player games~\cite{chenDengTeng09,daGoPa09}.
For this reason, new variants of Nash equilibria that are easier to compute were introduced, such as coarse correlated equilibria.

For a positive integer $n$, let $G=(\{1,\dots,n\},A=A_1\times \cdots\times A_n,u=(u_1,\dots,u_n))$ be a normal-form game of $n$ players with the actions sets $A_i$ and utility functions $u_i \colon A \to \mathbb{R}$.
For $\varepsilon > 0$, a probability distribution $P$ on $A$ is an \emph{$\varepsilon$-coarse correlated equilibrium} if $\mathbb{E}_{a \sim P}[u_i(a'_i;a_{-i}) - u_i(a)] \leq \varepsilon$ for every player $i \in [n]$ and each action $a'_i \in A_i$.

It is significantly simpler to compute Nash equilibria or coarse correlated equilibria with a small support, as those are closer to action profiles, which is witnessed by the classical quasipolynomial algorithm for approximating Nash equilibria by Lipton,  Markakis, and Mehta~\cite{lipMarMeh03}.
Thus, there has been a lot of interest in estimating the minimum support size of such equilibria; see, for example~\cite{babichenko14,barman18,lipMarMeh03}.
To capture this formally, we say that a probability distribution $P$ on $A$ is \emph{$k$-uniform} if there are $k$ action profiles $a^1,\dots,a^k$ such that $P = \frac{1}{k} \sum_{t=1}^k a^t$.

Babichenko, Barman, and Peretz~\cite{babichenko14} proved that every normal-form game of $n$ players, each with $m$ actions,  admits a $k$-uniform $\varepsilon$-coarse correlated equilibrium for all
\begin{equation}
\label{eq-babichenko}
k > \frac{2(\ln{m} + \ln{n})}{\varepsilon^2}.
\end{equation}
They also noted that this bound is asymptotically tight in terms of $m$, but establishing whether this logarithmic in $n$ dependence is tight remained an interesting open problem.
Actually, it was not even known whether $k$ cannot be a constant that depends only on $\varepsilon$ and not on~$n$.
To pinpoint this question, Babichenko, Barman, and Peretz~\cite{babichenko14} introduced the following problem.

\begin{problem}[\cite{babichenko14}]
Is there a $k = k(\varepsilon)$ that is independent of $n$, such that in every $n$-player 2-action game there exists $k$-uniform $\varepsilon$-correlated equilibrium?
\end{problem}

Balko and \v{C}\'{i}\v{z}ek~\cite{balCiz26} answered this problem negatively by proving the following result with the help of Bolzano. 

\begin{theorem}[\cite{balCiz26}]
\label{thm:gameTheory}
There is a constant $C>0$ such that for every $\varepsilon \in (0,1/2)$, there are arbitrarily large integers $n$ and normal-form games $G$ of $n$ players, each with two actions, with payoffs in $\{0,1\}$ such that every $k$-uniform $\varepsilon$-coarse correlated equilibrium in $G$ satisfies
\[k \geq \frac{C \cdot\log{n}}{\varepsilon^2\log{(1/\varepsilon)}}.\]    
\end{theorem}

Note that this result not only shows that $k$ has to depend on $n$, but also provides a lower bound that is asymptotically tight in $n$ as it matches the upper bound~\eqref{eq-babichenko} in terms of $n$.
The full proof appears in~\cite{balCiz26}, and we include it in Appendix~\ref{app:gameTheory} for completeness, as the paper~\cite{balCiz26} has not yet appeared.

Bolzano worked on this problem for four rounds, guided by a human researcher who pointed out the right research direction. The dependence on $n$ was settled in the first round; the remaining rounds were used to improve the dependence on $\varepsilon$.
The Bolzano research transcript is available at~\cite{bolzano_cce}.

\subsection{Equivalence of Weak and Strong Working Set Properties for Heaps -- Václav Rozhoň}
\label{sec:heap}

The setting for this subsection is the beyond worst-case theory of data structures, in particular heaps.
Two plausible definitions of a beyond worst-case heap occur in the literature -- the strong working set property from \cite{iacono} and the weak working set property from \cite{elmasry}.
The strong working set property in particular has been crucial in a recent line of research \cite{dijkstra,iacono}.

It was long assumed that the strong working set property was strictly stronger than the weak one. Surprisingly, they turn out to be equivalent. This was first proven by a human expert, but when given this task, Bolzano independently came up with a different proof and suggested a quantitative strengthening \Cref{eq:stronger_equivalence} that was not apparent from the original proof. The strengthening gives an even tighter picture of how close the two definitions are.

\paragraph{Formal results}

Consider heaps that support the operations \textsc{Insert} and \textsc{ExtractMin}. We study two notions of locality-sensitive cost for such data structures.

\begin{definition}[Weak working set property]
\label{def:weak_working_set}
A heap satisfies the \emph{weak working set property} if, for any sequence of $m$ operations, the total cost of serving them is $O(m + \sum_x \log(t'_x - t_x + 1))$, where the sum is over all extracted elements $x$, $t_x$ is the insertion time of $x$, and $t'_x$ is the extraction time of $x$.
\end{definition}

\begin{definition}[Strong working set property]
\label{def:strong_working_set}
A heap satisfies the \emph{strong working set property} if, for any sequence of $m$ operations, the total cost of serving them is $O(m + \sum_x \max_{t_x \le t < t'_x} \log(|W_{t,x}| + 1))$, where $W_{t,x}$ is the set of all elements that have been inserted after time $t_x$ and are still present at time $t$.
\end{definition}

\begin{theorem}
A heap has the weak working set property if and only if it has the strong one. In fact, the following holds for any $\varepsilon > 0$.
\begin{align}
    \sum_x \log L_x \le (1+\varepsilon) \sum_x \log K_x + O(m/\varepsilon)\label{eq:stronger_equivalence}
\end{align}
where $L_x = t'_x - t_x + 1$ is the lifetime of element $x$ and $K_x = \max_{t_x \le t < t'_x}(|W_{t,x}| + 1)$ is its strong working set cost.
\end{theorem}
The theorem is proven in \Cref{app:heap-wsp}. The Bolzano research transcript is available at~\cite{bolzano_heaps}.

\section*{Acknowledgements}

We thank Tomáš Gavenčiak and Vojtěch Rozhoň for helpful discussions and for their contributions to the development of the Bolzano system. 
JG, VR, and AZ were supported by the Czech Science Foundation (GA ČR), project No. 26-23599M.
RŠ was supported by grant no.25-16627S from the Czech Science Foundation (GA\v{C}R).
MB and MK were supported by grant no.25-17221S from the Czech Science Foundation (GA\v{C}R).

\bibliographystyle{alpha}
\bibliography{ref}

\appendix

\section{Structural Results on Multi-Slope Tilings}
\label{app:tiling}

This appendix contains the proofs of \Cref{thm:Main_Tiling} from \Cref{sec:tiling}.

\medskip

The periodicity in \Cref{thm:weakPTC} allows to simplify the problem by working in the group $G = \mathbb{Z} \times \mathbb{T}$, where $\mathbb T:=\mathbb R/\mathbb Z$, with counting measure on $\mathbb{Z}$ and normalized Lebesgue measure $\mu$ on $\mathbb{R}/\mathbb{Z}$.
Addition is componentwise, that is, $(m,\theta)+(n,\varphi) = (m+n,\,\theta+\varphi)$. All equalities, coverings, and disjointness statements are understood a.e.\ with respect to the product measure.

In this setup, we work with measurable sets $A \subset G$ of the form $A = \bigcup_{i=1}^\ell \{n_i\}\times I_i$ with $n_i \in \mathbb{Z}$ distinct and each $I_i \subset \mathbb{T}$ a half-open interval.
The \emph{integer support} of $A$ is $\operatorname{supp}(A) := \{n_i : 1 \le i \le \ell\}$, and the \emph{fiber} at $n$ is $A_n := \{t \in \mathbb{T} : (n,t) \in A\}$.
A set $T \subset G$ is a \emph{tiling of $G$ by $A$} if $A + T = G$ and the translates $\{A+t\}_{t \in T}$ are pairwise disjoint up to null sets.
This is denoted as $A\oplus T=G$.
If $A\oplus T=G$, then $A$ is called a \emph{tile}.

\medskip

Analogously to Definition~\ref{def:Column}, a set $A\subseteq G$ is a \emph{column tile} if there exist finite sets $C \subset \mathbb{Z}$ and $\Lambda \subset \mathbb{T}$ such that $A \subset C \times \mathbb{T}$ and the vertical translates $\{A+(0,\lambda) : \lambda \in \Lambda\}$ tile $C \times \mathbb{T}$.

\subsection{Auxiliary claims}

\begin{lemma}[Multi-coset affine arithmetic-splitting criterion]\label{lm:SplittCriterion}
    Let $q\ge 1$, and $H\subset q\mathbb Z$ be a finite set of distinct integers.
    Let $J_h\subset\mathbb T$ be measurable for each $h\in H$, and $\beta_r,\alpha_r\in\mathbb T$ for each $0\le r\le q-1$.
    Define
    $$A:=\bigcup_{h\in H}(\{h\}\times J_h)\subset G \text{ and } T_r:=\{(qk+r,\beta_r+k\alpha_r):k\in\mathbb Z\}\subset G \text{ for each } 0\le r\le q-1.$$
    Then the following are equivalent:
\begin{enumerate}
    \item $A\oplus T=G$ where $T=\bigcup_{r=0}^{q-1} T_r$,
    \item for every $r\in R$, one has the circle partition identity
   $$\bigsqcup_{h\in H}\Bigl(J_h-\frac{h}{q}\alpha_r\Bigr)=\mathbb T \quad \text{a.e.}$$
\end{enumerate}
\end{lemma}
\begin{proof}
First observe that if $1\le r\le q-1$ and $t=(qk+r,\beta_r+k\alpha_r)\in T_r$, then $A+t\subseteq (q\mathbb{Z}+r)\times \mathbb{T}$.
Hence, it is enough to understand each fixed $1\le r\le q-1$ separately.

\medskip

For $m\in \mathbb{Z}$, define the set
$$F_{r,m}=\{\theta\in\mathbb T:(r+qm,\theta)\in A+T_r\}.$$
A point $(r+qm,\theta)$ lies in $A+(qk+r,\beta_r+k\alpha_r)$ for some $k\in \mathbb{Z}$ if and only if there exists $h\in H$ such that
\begin{itemize}
    \item $q(m-k)=h$, and
    \item $\theta\in J_h+\beta_r+k\alpha_r$.
\end{itemize}
Because $h\in q\mathbb Z$, the above equation has the unique solution $k=m-\frac{h}{q}$.
Substituting this into the vertical condition gives $\theta\in \Bigl(J_h-\frac{h}{q}\alpha_r\Bigr)+\beta_r+m\alpha_r$.
This gives
\begin{equation}\label{eq:Partition}
    F_{r,m}=\left(\bigcup_{h\in H}\Bigl(J_h-\frac{h}{q}\alpha_r\Bigr)\right)+\beta_r+m\alpha_r.
\end{equation}
Now we are ready to show the equivalence.

\medskip

$(2)\Rightarrow(1)$.
Let $1\le r\le q-1$ and $m\in \mathbb{Z}$.
Since translation on $\mathbb T$ preserves Lebesgue measure and preserves pairwise disjointness a.e., the assumed partition identity implies $F_{r,m}=\mathbb T$ a.e., and the pieces contributing to \eqref{eq:Partition} are pairwise disjoint a.e.
As $m\in \mathbb{Z}$ was arbitrary, we get that $A\oplus T_r=(q\mathbb{Z}+r)\times \mathbb{T}$.
This proves (1) by the remark in the beginning of this proof.

\medskip

$(1)\Rightarrow(2)$.
By the remark in the beginning of the proof and the assumption, we have that \eqref{eq:Partition} reads as
\begin{equation*}
    \mathbb{T}=F_{r,0}=\left(\bigsqcup_{h\in H}\Bigl(J_h-\frac{h}{q}\alpha_r\Bigr)\right)+\beta_r.
\end{equation*}
for every $1\le r\le q-1$ and $m=0$.
Shifting by $-\beta_r$ gives (2).    
\end{proof}

\begin{proposition}[Fiberwise criterion for column tiles]\label{pr:Column}
Let $A=\bigcup_{i=1}^\ell \{n_i\}\times I_i \subseteq \mathbb{Z}\times \mathbb{T}$ be a column tile.
Then there is $k\in \mathbb{N}$ such that
$$\mu\left(A_{n_i}\right)=1/k$$
for every $1\le i\le \ell$.
\end{proposition}
\begin{proof}
By the definition, there are $C\subset \mathbb Z$ and a finite nonempty set $\Lambda\subset \mathbb T$ such that the translates $\{A+(0,\lambda):\lambda\in\Lambda\}$ are pairwise disjoint a.e., and $\bigsqcup_{\lambda\in\Lambda}(A+(0,\lambda))=C\times\mathbb T$.

For each $1\le i\le \ell$, observe that $\mathbb{T}=\bigsqcup_{\lambda\in \Lambda} A_{n_i}+\lambda$.
Therefore
$$1=\mu(\mathbb T)=\sum_{\lambda\in\Lambda}\mu(A_{n_i}+\lambda)=|\Lambda|\mu(A_{n_i})$$
as Lebesgue measure is translation-invariant.
\end{proof}

\subsection{Explicit construction}
Choose an irrational $\varepsilon\in (0,1/3)$ and define:
\begin{itemize}
    \item $\alpha=1-\varepsilon$,
    \item $I_0:=[\varepsilon,2\varepsilon)$, $I_2:=[2\varepsilon,1)$, $I_4:=[0,\varepsilon)$,
    \item $A_\alpha=(\{0\}\times I_0)\cup(\{2\}\times I_2)\cup(\{4\}\times I_4)\subset\mathbb Z\times\mathbb T$,
    \item $S_{1,\alpha}:=\{(2k,0):k\in\mathbb Z\}$ and $S_{2,\alpha}:=\{(2k+1,k\alpha):k\in\mathbb Z\}.$
\end{itemize}
We show that $A_\alpha\oplus S_\alpha=G$ where $S_\alpha=S_{1,\alpha}\sqcup S_{2,\alpha}$.
We verify (2) in Lemma~\ref{lm:SplittCriterion} with $q=2$, $\beta_0=\beta_1=0$, $\alpha_0=0$ and $\alpha_1=\alpha$.
For $r=0$, we have
$$\bigsqcup_{2i\in \{0,2,4\}}I_{2i}=[\varepsilon,2\varepsilon)\sqcup[2\varepsilon,1)\sqcup[0,\varepsilon)=\mathbb T,$$
and for $r=1$, we have
\begin{equation*}
    \begin{split}
        \bigsqcup_{2i\in \{0,2,4\}}\Bigl(I_{2i}-i\alpha\Bigr)= & \ [\varepsilon,2\varepsilon)\sqcup([2\varepsilon,1)-(1-\epsilon))\sqcup([0,\varepsilon)-2(1-\epsilon))\\
        = & \ [\varepsilon,2\varepsilon)\sqcup([0,\epsilon)\sqcup [3\epsilon,1))\sqcup([2\epsilon,3\varepsilon)\\
        = & \ \mathbb T,
    \end{split}
\end{equation*}
where $[2\varepsilon,1)-(1-\epsilon)=[0,\epsilon)\sqcup [3\epsilon,1)$ follows from the assumption that $\epsilon\in (0,1/3)$.

\medskip

Observe that, by Proposition~\ref{pr:Column}, $A$ is not a column tile as $\mu(A\cap (\{0\}\times \mathbb{T}))$ is not rational.

\subsection{Proof of Theorem~\ref{thm:Main_Tiling}}
Given an irrational $\alpha\in (2/3,1)$, consider the tile $A_\alpha\subseteq \mathbb{Z}\times \mathbb{T}$ described above, and define
$$\Omega_\alpha=\{(n,x)+(t,0):  x,t\in [0,1) \text{ and } (n,x)\in A_\alpha\}\subseteq \mathbb{R}^2.$$
It follows from the properties of $A_\alpha$ and $S_\alpha$ that $\Omega_\alpha$ does not tile a column, and that $T_\alpha=T_{1,\alpha}\sqcup T_{2,\alpha}$, where $T_{1,\alpha}=(2\mathbb{Z})\times \mathbb{Z}$ and $T_{2,\alpha}=\{(2k,k\alpha+\ell):k,\ell\in \mathbb{Z}\}$, have all the desired properties.

\section{Equivalence of Weak and Strong Working Set Properties}
\label{app:heap-wsp}

In this appendix we prove the equivalence between the weak and strong working set
properties stated in \Cref{sec:heap}. In fact, we prove the stronger quantitative estimate
\begin{equation}
\label{eq:stronger_equivalence_app}
\sum_x \log L_x
\;\le\;
(1+\varepsilon)\sum_x \log K_x + O(m/\varepsilon),
\end{equation}
valid for every $\varepsilon>0$, where the sum ranges over all extracted elements.

The easy point is that $K_x \le L_x$ for every element $x$, so the strong working set
bound always implies the weak one. The substance is the converse direction:
although an individual lifetime $L_x$ can be much larger than the corresponding strong
parameter $K_x$, this can only happen for relatively few elements at the same time.
The proof makes this precise via a packing argument: at any time $t$, there cannot be
many live elements with small strong parameter. Grouping elements according to the size of
$K_x$, this packing bound yields an upper bound on the total lifetime within each group,
and Jensen's inequality then controls the total contribution of
$\log(L_x/K_x)$.

\paragraph{Setup.}
Consider any sequence of $m$ heap operations, each either \textsc{Insert} of a fresh element
or \textsc{ExtractMin} removing a present element. Elements are distinct even if keys coincide.
For each element $x$ that is eventually extracted, let
$t_x$ be its insertion time, $t'_x$ its extraction time, and define its
\emph{lifetime}
\[
L_x := t'_x-t_x+1.
\]
For each time $t$ with $t_x \le t < t'_x$, let $W_{t,x}$ be the set of all elements
inserted after time $t_x$ that are still present immediately after operation $t$.
We also define
\[
K_x := \max_{t_x \le t < t'_x}\bigl(|W_{t,x}|+1\bigr).
\]
Thus $L_x$ is the weak working set parameter and $K_x$ is the strong one.

For a time $t$, let $A(t)$ denote the set of elements present immediately after operation $t$.
All logarithms are base $2$.

\begin{lemma}
\label{lem:lifetime-dominates}
For every extracted element $x$, we have $K_x \le L_x$.
\end{lemma}

\begin{proof}
Fix $t$ with $t_x \le t < t'_x$. Every element of $W_{t,x}$ was inserted after time $t_x$
and no later than time $t$, so $|W_{t,x}| \le t-t_x$. Hence
\[
|W_{t,x}|+1 \le (t-t_x)+1 \le (t'_x-t_x)+1 = L_x.
\]
Taking the maximum over all admissible $t$ gives $K_x \le L_x$.
\end{proof}

The next lemma is the key combinatorial input. It says that elements with small strong
parameter cannot overlap too much in time.

\begin{lemma}[Packing lemma]
\label{lem:congestion}
Fix a time $t$ and an integer $k \ge 1$. Then at most $k$ elements $x \in A(t)$ satisfy
$K_x \le k$.
\end{lemma}

\begin{proof}
Suppose for contradiction that there are $k+1$ such elements alive after operation $t$.
Let $x$ be the oldest among them, i.e.\ the one with minimum insertion time.
Then the other $k$ elements were all inserted after $x$ and are all present at time $t$,
hence they all belong to $W_{t,x}$. Therefore
\[
|W_{t,x}| \ge k,
\qquad\text{so}\qquad
K_x \ge |W_{t,x}|+1 \ge k+1,
\]
contradicting $K_x \le k$.
\end{proof}

We now prove the quantitative comparison \eqref{eq:stronger_equivalence_app}.

\begin{theorem}
\label{thm:stronger_equivalence_app}
For every $\varepsilon>0$ and every operation sequence of length $m$,
\[
\sum_x \log L_x
\;\le\;
(1+\varepsilon)\sum_x \log K_x + O(m/\varepsilon),
\]
where the sum ranges over all extracted elements $x$.
\end{theorem}

\begin{proof}
Fix $\varepsilon>0$. We may assume $\varepsilon \le 1$, since otherwise the claim is weaker
than the case $\varepsilon=1$. Let
\[
b := \left\lceil \frac{1}{\varepsilon}\right\rceil,
\]
so that $1/b \le \varepsilon$ and $b = O(1/\varepsilon)$.

We partition the extracted elements into coarse levels according to the size of $K_x$:
for each integer $j \ge 0$, let
\[
X_j := \{x : 2^{bj} \le K_x < 2^{b(j+1)}\},
\qquad
N_j := |X_j|.
\]

We first bound the total lifetime inside one level. Since an element $x$ is alive after
exactly the times $t=t_x,t_x+1,\dots,t'_x-1$, it contributes $L_x-1$ to
$\sum_t \mathbf{1}_{x \in A(t)}$. Therefore
\[
\sum_{x \in X_j}(L_x-1)
=
\sum_{t=1}^m |X_j \cap A(t)|.
\]
Now if $x \in X_j$, then $K_x < 2^{b(j+1)}$, hence $K_x \le 2^{b(j+1)}-1$.
Applying \Cref{lem:congestion} with $k=2^{b(j+1)}-1$, we get
\[
|X_j \cap A(t)| \le 2^{b(j+1)}-1 \le 2^{b(j+1)}
\qquad\text{for every } t.
\]
Summing over $t$ yields
\[
\sum_{x \in X_j}(L_x-1) \le m\cdot 2^{b(j+1)},
\]
and hence
\begin{equation}
\label{eq:level-area}
\sum_{x \in X_j} L_x
\le
m\cdot 2^{b(j+1)} + N_j
\le
2m\cdot 2^{b(j+1)}.
\end{equation}

We now estimate the excess
\[
G := \sum_x \log\!\left(\frac{L_x}{K_x}\right).
\]
For $x \in X_j$ we have $K_x \ge 2^{bj}$, and therefore
\[
\log\!\left(\frac{L_x}{K_x}\right)
\le
\log\!\left(\frac{L_x}{2^{bj}}\right).
\]
Set $y_x := L_x/2^{bj}$. By \Cref{lem:lifetime-dominates}, we have $y_x \ge 1$.
Also, by \eqref{eq:level-area},
\[
\sum_{x \in X_j} y_x
=
\sum_{x \in X_j} \frac{L_x}{2^{bj}}
\le
\frac{2m \cdot 2^{b(j+1)}}{2^{bj}}
=
2m\cdot 2^b.
\]
Since $\log$ is concave, Jensen's inequality gives
\[
\sum_{x \in X_j} \log\!\left(\frac{L_x}{K_x}\right)
\le
\sum_{x \in X_j} \log y_x
\le
N_j \log\!\left(\frac{2m\cdot 2^b}{N_j}\right).
\]
Summing over all levels,
\[
G
\le
\sum_j N_j \log\!\left(\frac{2m\cdot 2^b}{N_j}\right).
\]

Let $n := \sum_j N_j$ be the number of extracted elements, and let $p_j := N_j/n$.
Then
\[
G
\le
n \log\!\left(\frac{2m\cdot 2^b}{n}\right) + nH(p),
\]
where $H(p):= -\sum_j p_j \log p_j$ is the entropy of the distribution $(p_j)_j$.

We bound the two terms separately.

For the first term,
\[
n \log\!\left(\frac{2m\cdot 2^b}{n}\right)
=
n(b+1) + n\log(m/n).
\]
Since $n \le m$, the first summand is at most $m(b+1)$.
For the second, using the elementary inequality $\log z \le (z-1)/\ln 2$ for $z\ge 1$,
we get
\[
n\log(m/n) \le \frac{m}{\ln 2}.
\]
Hence
\begin{equation}
\label{eq:first-term}
n \log\!\left(\frac{2m\cdot 2^b}{n}\right) = O(mb).
\end{equation}

For the entropy term, compare $p$ with the geometric distribution
$q_j := 2^{-(j+1)}$ on $\{0,1,2,\dots\}$.
Nonnegativity of relative entropy gives
\[
H(p) \le -\sum_j p_j \log q_j = \sum_j p_j(j+1) = 1 + \sum_j p_j j,
\]
and thus
\[
nH(p) \le n + \sum_j N_j j.
\]
Now if $x \in X_j$, then $K_x \ge 2^{bj}$, so $\log K_x \ge bj$, i.e.
\[
j \le \frac{\log K_x}{b}.
\]
Therefore
\[
\sum_j N_j j
=
\sum_x j(x)
\le
\frac{1}{b}\sum_x \log K_x,
\]
where $j(x)$ denotes the unique level index such that $x \in X_{j(x)}$.
Since also $n \le m$, we obtain
\begin{equation}
\label{eq:entropy-term}
nH(p) \le m + \frac{1}{b}\sum_x \log K_x.
\end{equation}

Combining \eqref{eq:first-term} and \eqref{eq:entropy-term}, we conclude that
\[
G \le O(mb) + \frac{1}{b}\sum_x \log K_x.
\]
Using $1/b \le \varepsilon$ and $b=O(1/\varepsilon)$, this becomes
\[
G \le \varepsilon \sum_x \log K_x + O(m/\varepsilon).
\]
Finally,
\[
\sum_x \log L_x
=
\sum_x \log K_x + G
\le
(1+\varepsilon)\sum_x \log K_x + O(m/\varepsilon),
\]
as claimed.
\end{proof}

As an immediate consequence, the weak and strong working set properties are equivalent.

\begin{corollary}
\label{cor:wsp-equivalence}
A heap satisfies the weak working set property if and only if it satisfies the strong working set property.
\end{corollary}

\begin{proof}
If a heap satisfies the weak working set property, then its total cost on every operation
sequence is
\[
O\!\left(m + \sum_x \log L_x\right).
\]
Applying \Cref{thm:stronger_equivalence_app} with any fixed $\varepsilon$ (say $\varepsilon=1$),
we get
\[
\sum_x \log L_x = O\!\left(m + \sum_x \log K_x\right),
\]
and therefore the total cost is
\[
O\!\left(m + \sum_x \log K_x\right),
\]
which is exactly the strong working set property.

Conversely, if a heap satisfies the strong working set property, then by
\Cref{lem:lifetime-dominates} we have $K_x \le L_x$ for every extracted element $x$, hence
\[
\sum_x \log K_x \le \sum_x \log L_x.
\]
So a bound of the form
\[
O\!\left(m + \sum_x \log K_x\right)
\]
immediately implies
\[
O\!\left(m + \sum_x \log L_x\right),
\]
which is the weak working set property.
\end{proof}
\section{Partitioning under Function Preimage Constraints}
\label{app:function-partition}

In this appendix we prove the two main results stated in \Cref{sec:partition}:
the counterexample to the original conjecture (\Cref{thm:disproof}) and the positive
partition theorem under the stronger pairwise boundedness hypothesis
(\Cref{thm:pairwise-intro}). We also record two complementary observations:
under the stronger assumption of uniform $n$-boundedness the constant improves from
$2nk(k-1)+1$ to $nk(k-1)+1$, and even under pairwise $n$-boundedness one cannot hope
for a bound smaller than $2k-1$ in general.

The natural language for the problem is graph coloring. The following reformulation
will be used throughout.

\begin{lemma}[Conflict graph reformulation]
\label{lem:conflict-graph}
Let $E$ and $F$ be sets, and let $f_1,\dots,f_k:E\to F$ satisfy
$f_i(x)\neq f_j(x)$ for every $x\in E$ and every $i\neq j$.
Define the \emph{conflict graph} $G$ on vertex set $E$ by declaring distinct
$x,y\in E$ adjacent whenever
\[
f_p(x)=f_q(y)
\qquad\text{for some } p\neq q.
\]
Then for every integer $m\ge 1$, the following are equivalent:
\begin{enumerate}
    \item $E$ admits a partition $E=E_1\sqcup\cdots\sqcup E_m$ such that
    \[
    f_p(E_i)\cap f_q(E_i)=\emptyset
    \qquad\text{for every } i \text{ and every } p<q;
    \]
    \item the graph $G$ is properly $m$-colorable.
\end{enumerate}
\end{lemma}

\begin{proof}
Suppose first that $E=E_1\sqcup\cdots\sqcup E_m$ is such a partition.
If two distinct vertices $x,y\in E_i$ were adjacent in $G$, then for some
$p\neq q$ we would have $f_p(x)=f_q(y)$, and hence
$f_p(E_i)\cap f_q(E_i)\neq\emptyset$, a contradiction. Thus every $E_i$
is an independent set, so the partition defines a proper $m$-coloring of $G$.

Conversely, suppose that $c:E\to\{1,\dots,m\}$ is a proper $m$-coloring of $G$,
and let $E_i:=c^{-1}(i)$. If for some $i$ and some $p\neq q$ the sets
$f_p(E_i)$ and $f_q(E_i)$ were not disjoint, we could choose $x,y\in E_i$
with $f_p(x)=f_q(y)$. Then $x \ne y$ by assumptions, so $x$ and $y$ would be adjacent in $G$, contradicting
the fact that $E_i$ is a color class.
\end{proof}

\subsection{A counterexample to the original conjecture}

We now prove \Cref{thm:disproof}. The mechanism behind the counterexample is simple:
while for $k=2$ the assumption implies that the maximum degree in the conflict graph
is at most $2n$, for $k\ge 3$, the original assumption can be satisfied vacuously by means of a
``dummy'' function with many empty fibers, while the remaining two functions create
a highly chromatic conflict graph.

\begin{proof}[Proof of \Cref{thm:disproof}]
Fix $n\ge 1$ and $M\ge 1$. Let $m>2^M$, and define
\[
E:=\{(i,j):1\le i<j\le m\},
\qquad
F:=\{0,1,\dots,m\}.
\]
Define three functions on $E$ by
\[
f_1(i,j):=i,\qquad f_2(i,j):=j,\qquad f_3(i,j):=0.
\]

The pointwise distinctness condition holds: for every $(i,j)\in E$ we have
$i<j$ and $0\notin\{1,\dots,m\}$, so the three values
$f_1(i,j),f_2(i,j),f_3(i,j)$ are pairwise distinct.

The fiber size condition is also satisfied. Indeed, if $z\in\{1,\dots,m\}$, then
$f_3^{-1}(z)=\emptyset$, while for $z=0$ we have $f_1^{-1}(0)=\emptyset$.
Thus for every $z\in F$ there exists some $t\in\{1,2,3\}$ with
$|f_t^{-1}(z)|=0\le n$.

It remains to analyze the conflict graph $G$.
Since $f_3$ takes only the value $0$, while neither $f_1$ nor $f_2$ ever takes
the value $0$, no edge of $G$ is witnessed by an equality involving $f_3$.
Thus two distinct vertices $(i,j)$ and $(i',j')$ are adjacent if and only if
either
\[
f_2(i,j)=f_1(i',j')
\quad\Longleftrightarrow\quad
j=i',
\]
or
\[
f_1(i,j)=f_2(i',j')
\quad\Longleftrightarrow\quad
i=j'.
\]
This graph is known as the shift graph $S_m$. It is well-known, that $\chi(S_m)>M$; 
we include a simple proof to keep the treatment self-contained. 

Let $c$ be any proper coloring of $S_m$ using $r$ colors. For each $j\in\{1,\dots,m\}$ define
\[
A_j:=\{\ell:\text{ there exists } i<j \text{ with } c(i,j)=\ell\}.
\]
We claim that the sets $A_1,\dots,A_m$ are pairwise distinct subsets of $\{1, \dots, r\}$. Indeed, take
$1\le j<k\le m$, and let $\ell:=c(j,k)$. Then $\ell\in A_k$ by definition.
If also $\ell\in A_j$, then there exists $i<j$ with $c(i,j)=\ell$.
But in the shift graph the vertices $(i,j)$ and $(j,k)$ are adjacent, since
their second and first coordinates coincide, a contradiction, as $c$ is a coloring. 
Hence $\ell\notin A_j$, so $A_j\neq A_k$.

Thus we have $m$ distinct subsets of an $r$-element set, which implies
$m\le 2^r$. Since $m>2^M$, we conclude that $r>M$. Therefore
$\chi(S_m)>M$, and hence the conflict graph $G$ is not $M$-colorable.
By \Cref{lem:conflict-graph}, any partition satisfying the required disjointness
condition must use more than $M$ parts.
\end{proof}

The same construction immediately rules out the original conjecture for all larger
values of $k$.

\begin{corollary}
\label{cor:disproof-all-k}
For every $k\ge 3$, the original conjecture fails.
\end{corollary}

\begin{proof}
Starting from the construction above for $f_1,f_2,f_3$, add functions
$f_4,\dots,f_k$ whose ranges are disjoint from one another and from the range of
$f_1,f_2,f_3$. This preserves pointwise distinctness and does not remove any edges
from the conflict graph, so the graph still contains the shift graph $S_m$ and
therefore has arbitrarily large chromatic number.
\end{proof}

\subsection{The strengthened hypothesis}

The counterexample shows that the original assumption is too weak because it constrains
only individual fibers. A natural fix is to require smallness \emph{pairwise} across
the functions involved in a potential conflict. We now prove
\Cref{thm:pairwise-intro}.

\begin{proof}[Proof of \Cref{thm:pairwise-intro}]
Let $G$ be the conflict graph from \Cref{lem:conflict-graph}; we 
orient its edges so that every vertex has bounded indegree.
Consider an edge $xy$ of $G$. By definition, there exist indices $p\neq q$ such that
\[
f_p(x)=f_q(y)=:z.
\]
By pairwise $n$-boundedness,
\[
\min\bigl(|f_p^{-1}(z)|,\;|f_q^{-1}(z)|\bigr)\le n.
\]
If $|f_p^{-1}(z)|\le n$, orient the edge from $x$ to $y$; otherwise orient it from
$y$ to $x$. (If both inequalities hold, choose either orientation.)
We claim that every vertex has indegree at most $d:=nk(k-1)$. 
Fix a vertex $y\in E$. For an incoming edge $x\to y$, there exist indices $p\neq q$
such that
\[
f_p(x)=f_q(y)
\qquad\text{and}\qquad
|f_p^{-1}(f_q(y))|\le n.
\]
For any $y \in E$, the number of such triples $(p,q,x)$ is at most $k(k-1)n$, thus 
$\operatorname{indeg}(y)\le k(k-1)n = d$. 

Now any induced subgraph $H$ of $G$ has at most $d |V(H)|$ edges (counting arcs 
by their tails), thus average degree at most $2d$. Therefore, $G$ is $2d$-degenerate 
and therefore $(2d+1)$-colorable, so
\[
\chi(G)\le 2nk(k-1)+1.
\]
The desired partition now follows from \Cref{lem:conflict-graph}.
\end{proof}

Under even stronger hypothesis that every fiber of every function has size at most $n$,
one gets a better constant, because then the conflict graph has maximum degree (rather than merely maximum indegree) 
bounded by $d$. 

\begin{theorem}[Uniform $n$-boundedness]
\label{thm:uniform-bound}
Assume pointwise distinctness and suppose that
\[
|f_i^{-1}(z)|\le n
\qquad\text{for every } i\in\{1,\dots,k\}\text{ and every } z\in F.
\]
Then $E$ can be partitioned into $nk(k-1)+1$ parts with the required disjointness
property.
\end{theorem}

\begin{proof}
Let $G$ be the conflict graph. Fix a vertex $y\in E$.
For each ordered pair $(p,q)$ with $p\neq q$, define
\[
N_{p,q}(y):=\{x\in E\setminus\{y\}: f_p(x)=f_q(y)\}.
\]
Then $|N_{p,q}(y)|\le |f_p^{-1}(f_q(y))|\le n$. 
Since every neighbor of $y$ belongs to at least one such set,
\[
\deg_G(y)\le \sum_{p\neq q}|N_{p,q}(y)|\le nk(k-1).
\]
Thus $\Delta(G)\le nk(k-1)$, and therefore $G$ is
$(nk(k-1)+1)$-colorable. The claim follows from \Cref{lem:conflict-graph}.
\end{proof}

\subsection{Lower bounds}

The upper bounds above are not tight in general, but even under pairwise $n$-boundedness
one cannot hope for a bound independent of $k$.

\begin{proposition}
\label{prop:lower-bound}
For every $k\ge 2$ and every $n\ge 1$, there exist examples satisfying
pairwise $n$-boundedness for which at least $2k-1$ parts are required.
\end{proposition}

\begin{proof}
Let $E=F=\mathbb{Z}_{2k-1}$ and define for $i=1, \dots, k$ 
\[
f_i(x):=x+i.
\]
Each $f_i$ is a bijection, so in particular all fibers have size $1\le n$.
Hence the instance satisfies the stronger uniform $n$-boundedness condition.

We claim that the conflict graph is the complete graph $K_{2k-1}$.
Indeed, let $x\neq y$ and write $\delta:=y-x\in\mathbb{Z}_{2k-1}\setminus\{0\}$.
The set of ordered differences
$\{p-q: p\neq q\} = \{-(k-1),\dots,-1,1,\dots,k-1\}$ 
coincides with all nonzero residues modulo $2k-1$.
Hence there exist $p\neq q$ with $p-q=\delta$, 
which is equivalent to
\[
   f_p(x) = x+p=y+q = f_q(y). 
\]
Thus every two distinct vertices are adjacent.
Therefore the conflict graph is $K_{2k-1}$ and requires $2k-1$ colors.
\end{proof}

\section{Complexity of Optimization in KKOS Cultural Dynamics}
\label{app:opdiff}

In this appendix we prove the results stated in \Cref{sec:opdiff}. We first establish NP-hardness of the decision problem via a reduction from \textsc{Clique}, then prove membership in NP under rational input encoding, yielding NP-completeness. Finally, we characterize feasible supports on forests and give a polynomial-time algorithm.

Throughout, $G=(V,E)$ is an undirected graph, $A'$ is its adjacency matrix, $A=A'+I$ (self-loops added), $y \in \mathbb{R}_{\ge 0}^V$ is a distribution with $\|y\|_1=1$, and $c \in \mathbb{R}_{\ge 0}^V$ is a cost vector. A distribution $x \in \mathbb{R}_{\ge 0}^V$ with $\sum_v x_v=1$ is \emph{feasible} if for every edge $uv \in E$, whenever $x_u,x_v>0$ one has $(Ax)_u=(Ax)_v$. We write $\operatorname{supp}(x)=\{v : x_v>0\}$ and call $(Ax)_v$ the \emph{mass} at vertex~$v$.

\subsection{NP-hardness}

\begin{lemma}[Universal vertex forces clique support]
\label{lem:universal-clique}
Let $G$ be a graph with a universal vertex $z$. If $x$ is a feasible distribution with $z \in \operatorname{supp}(x)$, then $G[\operatorname{supp}(x)]$ is a clique.
\end{lemma}

\begin{proof}
Let $S=\operatorname{supp}(x)$.
Fix any $u \in S$ with $u \ne z$. Since $z$ is universal, $zu \in E$. Because $x_z,x_u>0$, feasibility gives $(Ax)_u=(Ax)_z$.
Now $z$ is adjacent to every vertex of $S$, and $A$ includes self-loops, so
$(Ax)_z = \sum_{v \in S} x_v = 1$.
Hence $(Ax)_u=1$. But
$(Ax)_u = \sum_{v \in S \cap N_G[u]} x_v$,
and every $x_v$ with $v \in S$ is positive. This sum equals $1 = \sum_{v \in S} x_v$ only if $S \subseteq N_G[u]$. Since this holds for every $u \in S$, the support induces a clique.
\end{proof}

\begin{lemma}[$\ell_1$ cost lower bound]
\label{lem:l1-lower-bound}
Assume $c_i=1$ for all $i$. Let $S \subseteq V$, and let $x$ be any distribution with $\operatorname{supp}(x) \subseteq S$. Then
$\|x-y\|_1 \ge 2(1-y(S))$,
where $y(S)=\sum_{i \in S} y_i$.
\end{lemma}

\begin{proof}
Because $x_i=0$ for $i \notin S$,
$\sum_{i \notin S} |x_i-y_i| = \sum_{i \notin S} y_i = 1-y(S)$.
Also, $\sum_{i \in S} x_i = 1$ and $\sum_{i \in S} y_i = y(S)$, so
$\sum_{i \in S} (x_i-y_i) = 1-y(S)$.
By the triangle inequality,
$\sum_{i \in S} |x_i-y_i| \ge 1-y(S)$.
Adding gives $\|x-y\|_1 \ge 2(1-y(S))$.
\end{proof}

\begin{theorem}[NP-hardness]
\label{thm:opdiff-hard}
The decision problem---given $G$, rational $y,c$, and rational threshold $B$, decide whether a feasible distribution $x$ with $\sum_v c_v |x_v-y_v| \le B$ exists---is NP-hard. Hardness holds even with $c_i=1$ for all~$i$, positive rational~$y$, and $G$ having a universal vertex.
\end{theorem}

\begin{proof}
We reduce from \textsc{Clique}. Let $(H,k)$ be an instance with $H=(U,F)$ and $m=|U|$. Construct $G$ by adding a universal vertex $z$ adjacent to all of~$U$. Set $c_i=1$ for all~$i$, $y_z=2/3$, $y_u=1/(3m)$ for each $u \in U$, and threshold $T_k = 2/3 - 2k/(3m)$.

\emph{Forward.} If $H$ has a clique $C$ of size $k$, set $S=\{z\}\cup C$. Then $G[S]$ is a clique, so any distribution on $S$ is feasible. Define $x_u=1/(3m)$ for $u \in C$, $x_u=0$ for $u \in U \setminus C$, $x_z=1-k/(3m)$. Its cost is $2(m-k)/(3m) = T_k$.

\emph{Backward.} Suppose $x$ is feasible with cost $\le T_k$. If $x_z=0$, the $z$-coordinate contributes $2/3$ and $\sum_{u \in U}|x_u-y_u| \ge 2/3$ (since $\sum_u x_u=1$ while $\sum_u y_u=1/3$), giving cost $\ge 4/3 > T_k$, a contradiction. So $z \in \operatorname{supp}(x)$, and by \Cref{lem:universal-clique}, $\operatorname{supp}(x)=\{z\}\cup C$ for a clique $C$ in $H$ of size $t$. By \Cref{lem:l1-lower-bound}, the cost is $\ge 2/3 - 2t/(3m)$, so $t \ge k$.
\end{proof}

\subsection{Membership in NP}

\begin{theorem}[NP-membership]
\label{thm:opdiff-np}
Under the standard binary encoding of rational input, the decision problem belongs to NP.
\end{theorem}

\begin{proof}
Let the instance be a yes-instance, with feasible distribution $x$ and $S=\operatorname{supp}(x)$.
Consider the linear program $P(S)$ with variables $u \in \mathbb{R}^V$, $t \in \mathbb{R}^V$, $\delta \in \mathbb{R}$:
\begin{alignat*}{2}
& u_i=0 &&\text{for } i \notin S, \\
& u_i \ge \delta &&\text{for } i \in S, \\
& 0 \le u_i \le 1 &&\text{for all } i, \qquad \textstyle\sum_i u_i = 1,\\
& (Au)_a=(Au)_b &&\text{for every edge } ab \text{ with } a,b \in S, \\
& t_i \ge u_i-y_i,\; t_i \ge y_i-u_i &&\text{for all } i, \qquad 0 \le t_i \le 1,\\
& \textstyle\sum_i c_i t_i \le B, &&\qquad 0 \le \delta \le 1.
\end{alignat*}
Maximize $\delta$. The point $(u,t,\delta)=(x,|x-y|,\min_{i \in S} x_i)$ is feasible with $\delta>0$, so the optimum is positive.

All coefficients are rational. An optimal basic feasible solution has coordinates whose length is polynomial in the input, so there exists an optimal solution $(u^*,t^*,\delta^*)$ with $\delta^*>0$, $\operatorname{supp}(u^*)=S$, and polynomial encoding length. Since $u^*$ satisfies all constraints, it is a feasible distribution with cost $\le B$. Thus every yes-instance has a polynomial-size rational certificate, verifiable in polynomial time.
\end{proof}

\subsection{Structural characterization}

\begin{lemma}[Dominated neighborhoods obstruct feasibility]
\label{lem:dominated}
Let $H$ be a graph. If there is an edge $uv \in E(H)$ with $N_H[u] \subsetneq N_H[v]$, then no strictly positive vector $x \in \mathbb{R}_{>0}^{V(H)}$ can satisfy $(Ax)_u=(Ax)_v$ for all edges.
\end{lemma}

\begin{proof}
We have $(Ax)_v-(Ax)_u = \sum_{w \in N_H[v]\setminus N_H[u]} x_w > 0$, since $N_H[v]\setminus N_H[u] \ne \emptyset$ and $x_w>0$ for all~$w$.
\end{proof}

\begin{proposition}[Chordal feasible supports are disjoint unions of cliques]
\label{prop:chordal-clique}
Let $G$ be a chordal graph. If $S \subseteq V$ is a feasible support, then every connected component of $G[S]$ is a clique.
\end{proposition}

\begin{proof}
Let $H$ be a connected component of $G[S]$. Since $G$ is chordal, $H$ is chordal. Restricting a feasible distribution to the vertices of $H$ gives a strictly positive vector satisfying the mass-equality constraints on~$H$.

If $H$ is not a clique, then $H$ has a simplicial vertex~$s$ that is not universal (every chordal graph has a simplicial vertex, and a non-clique has one that is not adjacent to all others). Pick any neighbor $v$ of~$s$ in~$H$. Since $s$ is simplicial, $N_H(s)$ is a clique, so every neighbor of~$s$ is also a neighbor of~$v$. Thus $N_H[s] \subseteq N_H[v]$, and since $s$ is not universal, $N_H[s] \subsetneq N_H[v]$. By \Cref{lem:dominated}, no strictly positive feasible vector exists on~$H$, a contradiction.
\end{proof}

\begin{proposition}[Feasible supports on forests are dissociation sets]
\label{prop:forest-dissociation}
Let $G$ be a forest and $S \subseteq V$ nonempty. Then there exists a feasible distribution with support exactly $S$ if and only if $G[S]$ has maximum degree at most~$1$.
\end{proposition}

\begin{proof}
$(\Rightarrow)$\; Suppose $x$ is feasible with $\operatorname{supp}(x)=S$. Let $H=G[S]$. If $H$ has a component with $\ge 3$ vertices, let $u$ be a leaf and $v$ its unique neighbor in that component. Then $v$ has another neighbor in the component, so $N_H[u] \subsetneq N_H[v]$. Since $(Ax)_w = \sum_{t \in N_H[w]} x_t$ for $w \in S$ (vertices outside $S$ carry zero mass), \Cref{lem:dominated} gives a contradiction.

$(\Leftarrow)$\; If $\Delta(H) \le 1$, set $x_i = 1/|S|$ for $i \in S$, $x_i=0$ otherwise. Every support edge $uv$ is an isolated $K_2$ in $H$, so $N_H[u]=N_H[v]=\{u,v\}$ and $(Ax)_u=(Ax)_v=2/|S|$.
\end{proof}

\subsection{Polynomial-time algorithm for forests}

\begin{lemma}[Closed-form cost for fixed support]
\label{lem:fixed-support-cost}
Let $S \subseteq V$ be nonempty. Among all distributions supported within $S$, the minimum $\ell_1$-cost is
\[
F(S) = \sum_{i \notin S} c_i y_i + (1-y(S))\min_{i \in S} c_i.
\]
\end{lemma}

\emph{Note:} the least-cost distribution above may \emph{not} be an equilibrium.

\begin{proof}
Let $m=1-y(S)$ and $c_*=\min_{i \in S} c_i$. For any distribution $x$ supported within $S$, we have $x_i=0$ for every $i \notin S$, so
$\sum_{i \notin S} c_i|x_i-y_i| = \sum_{i \notin S} c_i y_i$.
For the in-$S$ part, setting $d_i=x_i-y_i$, we have $\sum_{i \in S} d_i = m$, so
$\sum_{i \in S} c_i|d_i| \ge c_* \sum_{i \in S} |d_i| \ge c_* m$.

To attain equality, pick $r \in S$ with $c_r=c_*$ and set $x_i=y_i$ for $i \in S \setminus \{r\}$, $x_r=y_r+m$.
\end{proof}

\begin{theorem}[Polynomial-time algorithm for forests]
\label{thm:forest-algo}
If $G$ is a forest, the optimization problem is solvable in $O(n^2)$ time. The decision problem on forests is therefore in~$P$.
\end{theorem}

\begin{proof}
\emph{Step~1: reduction to dissociation sets.}\;
By \Cref{prop:forest-dissociation}, the feasible supports are exactly the nonempty dissociation sets (vertex sets $S$ with $\Delta(G[S]) \le 1$). Since dissociation sets are hereditary, $\operatorname{OPT} = \min\{F(S) : \emptyset \ne S \subseteq V,\; \Delta(G[S]) \le 1\}$.

\emph{Step~2: anchor at a minimum-cost vertex.}\;
Set $K = \sum_{i \in V} c_i y_i$.
For each $r \in V$, let $V_r = \{i \in V : c_i \ge c_r\}$, let $G_r=G[V_r]$, and assign weight $w_i^{(r)}=(c_i+c_r)y_i$ to each $i \in V_r$. Let $M_r$ be the maximum total weight of a dissociation set in $G_r$ containing~$r$.

Using \Cref{lem:fixed-support-cost}, for any nonempty dissociation set $S$ with anchor $r \in S$ minimizing $c_i$ over $S$:
$F(S) = K + c_r - \sum_{i \in S} w_i^{(r)} \ge K + c_r - M_r$.
Conversely, the dissociation set achieving $M_r$ gives $F(S_r)=K+c_r-M_r$.
Hence $\operatorname{OPT} = \min_{r \in V}(K+c_r-M_r)$.

\emph{Step~3: dynamic programming.}\;
Each $G_r$ is a forest. In the component containing~$r$, root the tree at~$r$. For each vertex~$v$ in a rooted subtree, define:
\begin{itemize}
\item $P_v$: max weight of a dissociation set in the subtree of $v$ with $v \notin S$,
\item $Q_v$: max weight with $v \in S$ and no child of $v$ in $S$,
\item $R_v$: max weight with $v \in S$ and exactly one child of $v$ in $S$.
\end{itemize}
Base case (leaf): $P_v=0$, $Q_v=w(v)$, $R_v=-\infty$.
For children $u_1,\ldots,u_t$:
\begin{align*}
P_v &= \textstyle\sum_{j=1}^t \max\{P_{u_j}, Q_{u_j}, R_{u_j}\}, \\
Q_v &= w(v) + \textstyle\sum_{j=1}^t P_{u_j}, \\
R_v &= w(v) + \textstyle\sum_{j=1}^t P_{u_j} + \max_{1 \le j \le t}(Q_{u_j}-P_{u_j}).
\end{align*}
(The $R_v$ formula uses the observation that selecting child $u_j$ adds $Q_{u_j}-P_{u_j}$ over the baseline. If $v \in S$ and child $u_j \in S$, then $u_j$ cannot have a selected child, so $u_j$ must be in state~$Q$.)

For the root component, $M_r$ uses $\max\{Q_r, R_r\}$ (the root must be selected). For other components, add $\max\{P_\rho, Q_\rho, R_\rho\}$ at each root~$\rho$. Each $M_r$ is computed in $O(n)$ time; iterating over all anchors takes $O(n^2)$.

Once a minimizing anchor and optimal dissociation set $S_r$ are found, the optimal distribution is $x_i=0$ for $i \notin S_r$, $x_i=y_i$ for $i \in S_r \setminus \{r\}$, $x_r=y_r+1-y(S_r)$.
\end{proof}

\section{Estimating the support size of $\varepsilon$-coarse correlated equilibria}
\label{app:gameTheory}

In this section, we prove Theorem~\ref{thm:gameTheory} from Section~\ref{sec:gameTheory}.
Recall that we want to show that there is a constant $C>0$ such that for every $\varepsilon \in (0,1/2)$, there are arbitrarily large integers $n$ and normal-form games $G$ of $n$ players, each with two actions, with payoffs in $\{0,1\}$ such that every $k$-uniform $\varepsilon$-coarse correlated equilibrium in $G$ satisfies
\[k \geq \frac{C \cdot\log{n}}{\varepsilon^2\log{(1/\varepsilon)}}.\] 

For an integer $s \geq 2$, we consider the following normal-form game $G=(S,A,u)$
 of $n=2\binom{s}{2}=s(s-1)$ players labeled by ordered pairs $(i,j)$ of integers satisfying $1 \leq i, j \leq s$ and $i \neq j$.
Each player $(i,j) \in S$ has the action set $A_{(i,j)}=\{-1,+1\}$.
For each $i\in [s]$, we also define the following subset of the players
\[
S_i=\{(i,j) \colon j \in [s] \setminus\{i\}\}
\]
and, for each $\{i,j\} \in \binom{[s]}{2}$, we set
\[S_{\{i,j\}} = (S_i \Delta S_j) \setminus \{(i,j),(j,i)\},\]
where $\Delta$ denotes the symmetric difference.
Note that $(i,j) \in S_i \setminus S_j$ and $(j,i) \in S_j \setminus S_i$. 
For an action profile $a = (a_{(i,j)})_{(i,j) \in S} \in A$ and a player $(i,j) \in S$, we define the payoffs by setting
\[
u_{(i,j)}(a) =
\begin{cases}
\frac{1+a_{(i,j)}a_{(j,i)}\chi_{S_{\{i,j\}}}(a)}{2} & \text{if } i<j, \\
\frac{1-a_{(i,j)}a_{(j,i)}\chi_{S_{\{i,j\}}}(a)}{2}  & \text{if } j>i,
\end{cases}
\]
where $\chi_T(a) = \prod_{(i',j') \in T}a_{(i',j')}$ for every subset $T$ of $S$.
Note that the payoffs attain values in $\{0,1\}$ as all actions are from $\{-1,+1\}$.

The following two auxiliary results show how the condition from the definition of $\varepsilon$-coarse correlated equilibrium can be used to obtain a bound that will eventually be used to estimate the Hamming distance between certain vectors.

\begin{lemma}
\label{lem-players1}
For every $\varepsilon > 0$, if $P$ is an $\varepsilon$-coarse correlated equilibrium in $G$, then $|\mathbb{E}_{a \sim P}[a_{(i,j)}a_{(j,i)}\chi_{S_{\{i,j\}}}(a)]| \leq 2\varepsilon$ for every player $(i,j) \in S$.
\end{lemma}
\begin{proof}
For every player $(i,j) \in S$ with $i<j$, it follows from the definition of the payoffs that
\[\mathbb{E}_{a \sim P}[u_{(i,j)}(a)] = \frac{1+\mathbb{E}_{a \sim P}[a_{(i,j)}a_{(j,i)}\chi_{S_{\{i,j\}}}(a)]}{2} .\]
For every fixed action $a'_{(i,j)} \in A_{(i,j)}$, the expected payoff of $(i,j)$ for playing $a'_{(i,j)}$ while others follow $P$ is 
\[\mathbb{E}_{a \sim P}[u_{(i,j)}(a'_{(i,j)};a_{-(i,j)})] = \frac{1+a'_{(i,j)}\mathbb{E}_{a \sim P}[a_{(j,i)}\chi_{S_{\{i,j\}}}(a)]}{2}.\]
Thus, since $a'_{(i,j)} \leq 1$, the maximum payoff that the player $(i,j)$ can get for playing a fixed action is $\frac{1+\mathbb{E}_{a \sim P}[a_{(j,i)}\chi_{S_{\{i,j\}}}(a)]}{2}$.

Since $P$ is an $\varepsilon$-coarse correlated equilibrium, we have \[\mathbb{E}_{a \sim P}[u_{(i,j)}(a'_{(i,j)};a_{-(i,j)})] \leq \mathbb{E}_{a \sim P}[u_{(i,j)}(a)] + \varepsilon.\]
Using the above expressions, we can rewrite this inequality as
\[\mathbb{E}_{a \sim P}[a_{(j,i)}\chi_{S_{\{i,j\}}}(a)] \leq \mathbb{E}_{a \sim P}[a_{(i,j)}a_{(j,i)}\chi_{S_{\{i,j\}}}(a)] + 2\varepsilon.\]
Since the left-hand side is non-negative, we obtain $\mathbb{E}_{a \sim P}[a_{(i,j)}a_{(j,i)}\chi_{S_{\{i,j\}}}(a)] \geq -2\varepsilon$.

Proceeding analogously for the player $(j,i)$, we can derive $\mathbb{E}_{a \sim P}[a_{(i,j)}a_{(j,i)}\chi_{S_{\{i,j\}}}(a)] \leq 2\varepsilon$.
Putting both inequalities together, we get the bound $|\mathbb{E}_{a \sim P}[a_{(i,j)}a_{(j,i)}\chi_{S_{\{i,j\}}}(a)]| \leq 2\varepsilon$.
\end{proof}

\begin{lemma}
\label{lem-players2}
For every $\varepsilon>0$, each pair $\{i,j\} \in \binom{[s]}{2}$, and every $\varepsilon$-coarse correlated equilibrium $P$ in $G$, we have
\[|\mathbb{E}_{a \sim P}[\chi_{S_i \Delta S_j}(a)]| \leq 2\varepsilon.\]
\end{lemma}
\begin{proof}
Fix a pair $\{i,j\} \in \binom{[s]}{2}$ and assume without loss of generality that $i<j$.
By applying Lemma~\ref{lem-players1} for the player $(i,j)$, we obtain $|\mathbb{E}_{a \sim P}[a_{(i,j)}a_{(j,i)}\chi_{S_{\{i,j\}}}(a)]| \leq 2\varepsilon$.
By the choice of $S_{\{i,j\}}$, we have $S_{\{i,j\}} = (S_i \Delta S_j) \setminus \{(i,j),(j,i)\},$ where $(i,j) \in S_i \setminus S_j$ and $(j,i) \in S_j \setminus S_i$.
Thus, $a_{(i,j)}a_{(j,i)}\chi_{S_{\{i,j\}}}(a) = \chi_{S_i \Delta S_j}(a)$.
It then follows that 
\[
|\mathbb{E}_{a \sim P}[\chi_{S_i \Delta S_j}(a)]|=|\mathbb{E}_{a \sim P}[a_{(i,j)}a_{(j,i)}\chi_{S_{\{i,j\}}}(a)]| \leq 2\varepsilon.
\]
\end{proof}

We now use the previous two results to construct a set of vectors with pairwise large Hamming distances that will eventually provide the desired lower bound on the support.

\begin{lemma}
\label{lem-players3}
For a positive integer $k$ and a fixed real number $\varepsilon\in (0,1/2)$, let $P$ be a $k$-uniform probability distribution on $A$.
If $R_1,\dots,R_s$ are subsets of $S$ such that $|\mathbb{E}_{a \sim P}[\chi_{R_\ell \Delta R_{\ell'}}(a)]|<2\varepsilon$ for every pair $\{\ell,\ell'\} \in \binom{[s]}{2}$, then $k \in \Omega(\log{s}/(\varepsilon^2\log{(1/\varepsilon)}))$.
\end{lemma}
\begin{proof}
Since $P$ is a $k$-uniform probability distribution on $A$, there are action profiles $a^1,\dots,a^k$ from $A$ such that $P= \frac{1}{k} \sum_{t=1}^k a^t$.
For each $\ell \in [s]$, we define the vector
\[v_\ell = (\chi_{R_\ell}(a^1),\dots,\chi_{R_\ell}(a^k)) \in \{-1,+1\}^k.\]

We now show that $|v_\ell \cdot v_{\ell'}| < 2\varepsilon k$  for every pair $\{\ell,\ell'\} \in \binom{[s]}{2}$.
It follows from the definition of $\chi_T(a)$ that
\[\chi_{R_\ell \Delta R_{\ell'}}(a^t) = \chi_{R_\ell}(a^t) \cdot \chi_{R_{\ell'}}(a^t).\]
Since $P= \frac{1}{k} \sum_{t=1}^k a^t$, we thus obtain
\[\frac{1}{k}|v_\ell \cdot v_{\ell'}|=\left|\frac{1}{k} \sum_{t=1}^k\chi_{R_\ell}(a^t) \cdot \chi_{R_{\ell'}}(a^t)\right|=\left|\frac{1}{k} \sum_{t=1}^k \chi_{R_\ell \Delta R_{\ell'}}(a^t)\right|=|\mathbb{E}_{a \sim P}[\chi_{R_\ell \Delta R_{\ell'}}(a)]|<2\varepsilon.\]

We can use this fact to show that the Hamming distance $d_H(v_{\ell},v_{\ell'})$ of any distinct vectors $v_\ell$ and $v_{\ell'}$ lies in $\left(\frac{(1-2\varepsilon)k}{2},\frac{(1+2\varepsilon)k}{2}\right)$.
Since $v_\ell,v_{\ell'}$ lie in $\{-1,+1\}^k$, we have
\begin{align*}
v_\ell \cdot v_{\ell'} &= |\{i \in[k]\colon (v_\ell)_i = (v_{\ell'})_i\}| - |\{i \in[k]\colon (v_\ell)_i \neq (v_{\ell'})_i\}|\\
&= d_H(v_{\ell},v_{\ell'}) - (k-d_H(v_{\ell},v_{\ell'})) = 2d_H(v_{\ell},v_{\ell'})-k.
\end{align*}
Thus, since $|v_\ell \cdot v_{\ell'}| < 2\varepsilon k$, we have $d_H(v_{\ell},v_{\ell'}) \in \left(\frac{(1-2\varepsilon)k}{2},\frac{(1+2\varepsilon)k}{2}\right)$.

Now, a standard sphere-packing argument gives the desired lower bound on $k$.
Let $V=\sum_{r=0}^{(1-2\varepsilon)k/4}\binom{k}{r}$.
If $B_\ell$ is the set of vectors from $\{-1,+1\}^k$ that are at Hamming distance at most $(1-2\varepsilon)k/4$ from $v_{\ell}$ with $\ell \in [s]$, then it follows from the lower bound on the Hamming distance that the sets $B_1,\dots,B_s$ are pairwise disjoint.
Since $B_\ell=V$ for every $\ell$, we obtain $sV \leq 2^k$.
If $\varepsilon$ is fixed, then the expression $k \geq \log{(sV)}$ can be expressed as $k \in \Omega(\log{s}/(\varepsilon^2\log{(1/\varepsilon)}))$; see~\cite{alon92}.
\end{proof}

We are now ready to prove Theorem~\ref{thm:gameTheory}

\begin{proof}[Proof of Theorem~\ref{thm:gameTheory}]
Choose a fixed $\varepsilon \in (0,1/2)$.
Let $P$ be a $k$-uniform $\varepsilon$-coarse correlated equilibrium in $G$.
By Lemma~\ref{lem-players2}, the sets $S_1,\dots,S_s$ satisfy $|\mathbb{E}_{a \sim P}[\chi_{S_i \Delta S_j}(a)]| \leq 2\varepsilon$ for every $\{i,j\} \in \binom{[s]}{2}$.
The assumptions of Lemma~\ref{lem-players3} are thus satisfied for the sets $S_1,\dots,S_s$, and we obtain $k \in \Omega(\log{s}/(\varepsilon^2\log{(1/\varepsilon)}))$.
Since $n=s(s-1)$, we also have $k \in \Omega(\log{n}/(\varepsilon^2\log{(1/\varepsilon)}))$.
\end{proof}

\section{Interleaving and Wilber's First Lower Bound}
\label{app:wilber-merge}

In this appendix we prove \Cref{thm:wilber-merge} and deduce
\Cref{cor:wilber-merge}.

\paragraph{Dyadic setup.}
We assume throughout that $n$ is a power of two; an arbitrary $n$ can be
handled by padding to the next power of two, which changes the right-hand
side of \Cref{thm:wilber-merge} by at most constant factors. Let
$\mathcal{D}_n$ be the family of internal dyadic intervals of $[n]$; each
$I \in \mathcal{D}_n$ has two children $L(I)$ and $R(I)$.

For a sequence $S=(s_1,\dots,s_t)$ over $[n]$ and an interval
$I\in\mathcal{D}_n$, let $S|_I$ denote the subsequence of terms lying in $I$,
and let $\alpha_I(S)$ be the number of consecutive pairs in $S|_I$ whose two
endpoints lie in \emph{different} children of $I$. Wilber's first bound is
\[
\Wil(S) \;:=\; \sum_{I\in\mathcal{D}_n}\alpha_I(S).
\]

Let $Z$ be a two-colored sequence over $[n]$ (colors: \textsf{red} and
\textsf{blue}). To avoid notational clash with the right child $R(I)$, we
write $Z_{r}$ and $Z_{b}$ for the red and blue subsequences of $Z$ ---
these are the sequences denoted $R$ and $B$ in the statement of
\Cref{thm:wilber-merge}. The alternating merge $X \shuffle Y$ is the special
case with the $x_i$'s red and the $y_i$'s blue.

\begin{proof}[Proof of \Cref{thm:wilber-merge}]
Fix $I\in\mathcal{D}_n$. Call a consecutive pair of $Z|_I$ lying in opposite
children of $I$ an \emph{alternation at $I$}; classify it as
\emph{monochromatic} if its two endpoints share a color, \emph{bichromatic}
otherwise. Write $M_I$ for the number of bichromatic alternations at $I$, so
\[
\alpha_I(Z) \;=\; \#\{\text{monochromatic alternations at }I\} + M_I.
\]
We bound the two parts separately.

\paragraph{Monochromatic alternations.}
A consecutive monochromatic pair in $Z|_I$ is still consecutive in the
corresponding same-color restriction $Z_r|_I$ or $Z_b|_I$, so
$\#\{\text{monochromatic alternations at }I\} \le \alpha_I(Z_r) + \alpha_I(Z_b)$.
Summing over $I$,
\begin{equation}
\label{eq:wmerge-mono}
\sum_{I\in\mathcal{D}_n} \#\{\text{monochromatic alternations at }I\}
\;\le\; \Wil(Z_r) + \Wil(Z_b).
\end{equation}

\paragraph{Bichromatic alternations via telescoping.}
For a sequence $S$, let $C(S)$ denote the number of adjacent color changes in
$S$. Every adjacent color change in $Z|_I$ either has both endpoints in the
same child of $I$, or straddles the two children --- in the latter case it is
exactly a bichromatic alternation at $I$. Writing $C_{LL}(I)$ (resp.\
$C_{RR}(I)$) for the color changes of $Z|_I$ whose endpoints both lie in
$L(I)$ (resp.\ $R(I)$),
\begin{equation}
\label{eq:wmerge-split}
C(Z|_I) \;=\; C_{LL}(I) + C_{RR}(I) + M_I.
\end{equation}
Restricting further from $Z|_I$ to $Z|_{L(I)}$ only creates new adjacencies:
every color change in $Z|_{L(I)}$ is either inherited from an adjacent
same-child color change already counted by $C_{LL}(I)$, or is a pair $(u,v)$
adjacent in $Z|_{L(I)}$ but not in $Z|_I$. Let $\New_L(I)$ be the set of such
newly-created color changes, and define $\New_R(I)$ symmetrically. Then
\begin{equation}
\label{eq:wmerge-new}
C(Z|_{L(I)}) = C_{LL}(I) + |\New_L(I)|,
\qquad
C(Z|_{R(I)}) = C_{RR}(I) + |\New_R(I)|.
\end{equation}
Subtracting \eqref{eq:wmerge-new} from \eqref{eq:wmerge-split},
\begin{equation}
\label{eq:wmerge-MI}
M_I \;=\; C(Z|_I) - C(Z|_{L(I)}) - C(Z|_{R(I)})
         + |\New_L(I)| + |\New_R(I)|.
\end{equation}
We sum \eqref{eq:wmerge-MI} over $I \in \mathcal{D}_n$. The children of
dyadic intervals range over every non-root dyadic interval and every leaf
interval exactly once, so
\[
\sum_{I\in\mathcal{D}_n}\!\bigl(C(Z|_{L(I)}) + C(Z|_{R(I)})\bigr)
\;=\;
\sum_{\substack{J\in\mathcal{D}_n\\ J\neq [n]}}\! C(Z|_J)
\;+\;
\sum_{\ell\text{ leaf}} C(Z|_\ell).
\]
The first sum on the right telescopes against $\sum_{I\in\mathcal{D}_n} C(Z|_I)$;
the second is non-negative. What remains is the root term
$C(Z|_{[n]}) \le |Z|-1$, giving
\begin{equation}
\label{eq:wmerge-telescope}
\sum_{I\in\mathcal{D}_n} M_I
\;\le\; |Z| + \sum_{I\in\mathcal{D}_n}\bigl(|\New_L(I)| + |\New_R(I)|\bigr).
\end{equation}

\paragraph{Charging new color changes to alternations in the monochromatic restrictions.}
Let $\mathcal{A}_I^{r}$ (resp.\ $\mathcal{A}_I^{b}$) be the set of
alternations at $I$ in $Z_r|_I$ (resp.\ $Z_b|_I$) and set
\[
\mathcal{A}_I \;:=\; \mathcal{A}_I^{r} \sqcup \mathcal{A}_I^{b},
\qquad
|\mathcal{A}_I| \;=\; \alpha_I(Z_r) + \alpha_I(Z_b).
\]
Note that $\mathcal{A}_I$ is, in general, strictly larger than the set of
monochromatic alternations at $I$ bounded in \eqref{eq:wmerge-mono}: a pair
consecutive in $Z_r|_I$ need not be consecutive in $Z|_I$. We construct a map
$\New_L(I) \cup \New_R(I) \to \mathcal{A}_I$ under which each element of
$\mathcal{A}_I$ receives at most two preimages, establishing
\begin{equation}
\label{eq:wmerge-charge}
|\New_L(I)| + |\New_R(I)|
\;\le\; 2\,|\mathcal{A}_I|
\;=\; 2\,\alpha_I(Z_r) + 2\,\alpha_I(Z_b).
\end{equation}

\medskip\noindent\emph{The charge map.}
Take $(u,v) \in \New_L(I)$. Since $u,v$ are adjacent in $Z|_{L(I)}$ but not in
$Z|_I$, all points of $Z|_I$ strictly between $u$ and $v$ form a non-empty
block $M \subseteq R(I)$. Because $\col(u) \neq \col(v)$, exactly one of the
following holds:
\begin{itemize}[nosep]
    \item[\textbf{(a)}] $M$ contains a point of color $\col(u)$. Let $x$ be
    the \emph{first} such point and charge $(u,v)$ to the pair $u \to x$.
    \item[\textbf{(b)}] Every point of $M$ has color $\col(v)$. Let $y$ be
    the \emph{last} point of $M$ and charge $(u,v)$ to $y \to v$.
\end{itemize}
In either case the charged pair lies in $\mathcal{A}_I$: its endpoints lie
in opposite children of $I$, and by the choice of $x$ or $y$ no point of the
shared color lies between them in $Z|_I$, so the pair is consecutive in the
corresponding monochromatic restriction $Z_r|_I$ or $Z_b|_I$. Define the
charging from $\New_R(I)$ symmetrically.

\medskip\noindent\emph{At most one preimage from $\New_L(I)$ per target.}
Fix an alternation $a \to b \in \mathcal{A}_I$.

If $a \in L(I)$ and $b \in R(I)$, a charge from $\New_L(I)$ to $a \to b$ can
only arise via case (a), so the contributing pair $(u,v)\in \New_L(I)$
satisfies $u = a$ and $x = b$. Since $(u,v)$ is adjacent in $Z|_{L(I)}$,
the point $v$ is forced to be the next element of $Z|_{L(I)}$ after $u$.
Hence there is at most one such preimage.

If $a \in R(I)$ and $b \in L(I)$, symmetrically the charge must arise via
case (b), so the contributing pair satisfies $y = a$ and $v = b$. Again
$(u,v) \in \New_L(I)$ forces $u$ to be the previous element of $Z|_{L(I)}$
before $v$, so again there is at most one such preimage.

By symmetry, each $a \to b$ also receives at most one preimage from
$\New_R(I)$. This proves \eqref{eq:wmerge-charge}.

\paragraph{Conclusion.}
Summing \eqref{eq:wmerge-charge} over $I$ and combining with
\eqref{eq:wmerge-telescope},
\[
\sum_{I\in\mathcal{D}_n} M_I
\;\le\; |Z| + 2\,\Wil(Z_r) + 2\,\Wil(Z_b).
\]
Together with \eqref{eq:wmerge-mono},
\[
\Wil(Z) = \sum_{I\in\mathcal{D}_n} \alpha_I(Z)
\;\le\; \bigl(\Wil(Z_r) + \Wil(Z_b)\bigr) + \sum_{I\in\mathcal{D}_n} M_I
\;\le\; 3\,\Wil(Z_r) + 3\,\Wil(Z_b) + |Z|.
\qedhere
\]
\end{proof}

\begin{proof}[Proof of \Cref{cor:wilber-merge}]
Since $X$ and $Y$ can be served at cost $C_1$ and $C_2$, we have
$\OPT(X) \le C_1$ and $\OPT(Y) \le C_2$. Wilber's lower
bound~\cite{wilber} gives $\Wil(X) = O(C_1)$ and $\Wil(Y) = O(C_2)$, so by
\Cref{thm:wilber-merge},
\[
\Wil(X \shuffle Y) \;\le\; 3\,\Wil(X) + 3\,\Wil(Y) + 2m
\;=\; O(C_1 + C_2 + m).
\]
Applying a Tango-tree upper bound of the form~\cite{tango}
\[
\mathrm{BST\text{-}cost}(Z) \;=\; O\bigl((\Wil(Z) + |Z|)\log\log n + n\bigr)
\]
to $Z = X \shuffle Y$ with $|Z| = 2m$ gives
\[
\mathrm{BST\text{-}cost}(X \shuffle Y)
\;=\; O\bigl((C_1 + C_2 + m)\log\log n + n\bigr),
\]
which is the claimed bound.
\end{proof}

\end{document}